\documentclass{article}
\usepackage[margin=1.5in]{geometry}
\usepackage[utf8]{inputenc} 
\usepackage[T1]{fontenc}    
\usepackage{hyperref}       
\usepackage{url}            
\usepackage{booktabs}       
\usepackage{amsfonts}       
\usepackage{nicefrac}       
\usepackage{microtype}      
\usepackage[english]{babel}
\usepackage{amsmath,amsthm}
\usepackage{amssymb,mathrsfs}
\usepackage{graphicx}
\usepackage{tikz}
\usepackage{subfig}
\usepackage{enumerate}
\usepackage{verbatim}
\usepackage{xcolor}
\usepackage{soul}
\usepackage[normalem]{ulem}
\usepackage{flushend}
\usepackage{cprotect}
\usepackage{tcolorbox}
\usepackage{dsfont}
\usepackage{algorithm}
\usepackage[noend]{algpseudocode}
\usepackage{algorithmicx}
\usepackage[noend]{algpseudocode}
\usepackage{wrapfig}
\usepackage{lipsum}
\usepackage{cutwin}
\usepackage{caption}
\usepackage{subfig}
\usepackage{bigints}

\DeclareMathOperator*{\argmin}{arg\,min}

\usepackage{tikz}
\usetikzlibrary{arrows,positioning,fit,backgrounds}
\usetikzlibrary{calc}

\usetikzlibrary{arrows,positioning,fit,backgrounds,shapes}
\usetikzlibrary{calc}

\newtheorem{definition}{Definition}
\newtheorem{remark}{Remark}
\newtheorem{theorem}{Theorem}

\newtheorem{lemma}{Lemma}

\newtheorem{proposition}{Proposition}
\newtheorem{corollary}{Corollary}
\newtheorem{example}{Example}

\newcommand{\be}{\begin{equation}}
\newcommand{\ee}{\end{equation}}
\newcommand{\ben}{\begin{enumerate}}
\newcommand{\een}{\end{enumerate}}
\newcommand{\bea}{\begin{eqnarray}}
\newcommand{\eea}{\end{eqnarray}}
\newcommand{\bean}{\begin{eqnarray*}}
\newcommand{\eean}{\end{eqnarray*}}
\newcommand{\pr}{\mathbb{P}}

\newcommand{\E}{\mathbb{E}}

\newcommand{\pa}{\mathcal{P}}

\newcommand{\X}{\mathcal{X}}

\newcommand{\gen}{\mathrm{gen}}

\newcommand{\risk}[0]{\mathrm{risk}}
\newcommand{\W}[0]{\mathcal{W}}
\newcommand{\Z}[0]{\mathsf{Z}}

\newcommand{\G}{\mathcal{G}}

\title{Chaining Meets Chain Rule:\\
Multilevel Entropic Regularization and Training of Neural Nets}

\author{
  Amir R.~Asadi \\
  Princeton University\\
  \texttt{aasadi@princeton.edu} \\
   \and
   Emmanuel Abbe \\
   EPFL \\
   \texttt{emmanuel.abbe@epfl.ch} \\
}

\begin{document}

\maketitle

\begin{abstract}
We derive generalization and excess risk bounds for neural nets using a family of complexity measures based on a multilevel relative entropy. The bounds are obtained by introducing the notion of generated hierarchical coverings of neural nets and by using the technique of chaining mutual information introduced in Asadi et al.\ NeurIPS'18. The resulting bounds are algorithm-dependent and exploit the multilevel structure of neural nets. This, in turn, leads to an empirical risk minimization problem with a multilevel entropic regularization. The minimization problem is resolved by introducing a multi-scale generalization of the celebrated Gibbs posterior distribution, proving that the derived distribution achieves the unique minimum. This leads to a new training procedure for neural nets with performance guarantees, which exploits the chain rule of relative entropy rather than the chain rule of derivatives (as in backpropagation). To obtain an efficient implementation of the latter, we further develop a multilevel Metropolis algorithm simulating the multi-scale Gibbs distribution, with an experiment for a two-layer neural net on the MNIST data set. 
\end{abstract}

\section{Introduction}
We introduce a family of complexity measures for the hypotheses of neural nets, based on a multilevel relative entropy. These complexity measures take into account the multilevel structure of neural nets, as opposed to the classical relative entropy (KL-divergence) term derived from PAC-Bayesian bounds \cite{catoni2007pac} or mutual information bounds \cite{russo2016controlling, xu2017information}. We derive these complexity measures by combining the technique of chaining mutual information (CMI) \cite{asadi2018chaining}, an algorithm-dependent extension of the classical chaining technique paired with the mutual information bound \cite{russo2016controlling}, with the multilevel architecture of neural nets. It is observed in this paper that if a neural net is regularized in a multilevel manner as defined in Section \ref{multilevel regularization section}, then one can readily construct hierarchical coverings with controlled diameters for its hypothesis set, and exploit this to obtain new multi-scale and algorithm-dependent generalization bounds and, in turn, new regularizers and training algorithms. The effect of such multilevel regularizations on the representation ability of neural nets has also been recently studied in  \cite{hardt2016identity, bartlett2018representing} for the special case where layers are nearly-identity functions as for ResNets \cite{he2016deep}. Here, we demonstrate the advantage of multilevel architectures by showing how one can obtain accessible hierarchical coverings for their hypothesis sets, introducing the notion of architecture-generated coverings in Section \ref{generated sequence of partitions section}. Then we derive our generalization bound for arbitrary-depth feedforward neural nets via applying the CMI technique directly on their hierarchical sequence of generated coverings. Although such a sequence of coverings may not give the tightest possible generalization bound, it has the major advantage of being easily accessible, and hence can be exploited in devising multilevel training algorithms. Designing training algorithms based on hierarchical coverings of hypothesis sets has first been studied in \cite{cesa1999prediction}, and has recently regained traction in e.g.\ \cite{gaillard2015chaining, cesa2017algorithmic}, all in the context of online learning and prediction of individual sequences. With such approaches, hierarchical coverings are no longer viewed merely as methods of proof for generalization bounds: they further allow for algorithms achieving low statistical error. 

In our case, the derived generalization bound puts forward a multilevel relative entropy term (see Definition \ref{multilevel relative entropy definition}). We then turn to minimizing the empirical error with this induced regularization, called here the multilevel entropic regularization. Interestingly, we can solve this minimization problem exactly, obtaining a multi-scale generalization of the celebrated Gibbs posterior distribution; see Sections \ref{Gen bounds section main} and \ref{MT section}. The  target distribution is obtained in a backwards manner by successive marginalization and tilting of distributions, as described in the Marginalize-Tilt algorithm introduced in Section \ref{MT section}. Unlike the classical Gibbs distribution, its multi-scale counter-part possesses a \emph{temperature vector} rather than a global temperature. We then present a multilevel training algorithm by simulating our target distribution via a multilevel Metropolis algorithm introduced for a two layer net in Section \ref{multilevel training section}.  In contrast to the celebrated back-propagation algorithm which exploits the chain rule of derivatives, our target distribution and its simulated version are derived from the chain rule of relative entropy, and take into account the interactions between different scales of the hypothesis sets of neural nets corresponding to different depths.

This paper introduces the new concepts and main results behind this alternative approach to training neural nets. Many directions emerge from this approach, in particular for its applicability. It is worth noting that Markov chain Monte Carlo (MCMC) methods are known to often better cope with non-convexity issues than gradient descent approaches, since they are able to backtrack from local minima \cite{geman1987stochastic}. Furthermore, in contrast to gradient descent, MCMC methods take into account parameter uncertainty that helps preventing overfitting \cite{welling2011bayesian}. However, compared to gradient based methods, these methods are typically computationally more demanding. 

\paragraph{Further related literature} Information-theoretic approaches to statistical learning have been studied in the PAC-Bayesian theory; see \cite{catoni2007pac, guedj2019primer, audibert2004pac} and references therein, and via the recent mutual information bound in e.g. \cite{russo2016controlling,xu2017information, raginsky2016information, jiao2017dependence, pensia2018generalization, bassily2017learners, bu2019tightening}. Deriving generalization bounds for neural nets, based on the PAC-Bayesian theory, has been the focus of recent studies such as \cite{dziugaite2017computing, neyshabur2017pac, zhou2018non, dziugaite2018data}. The statistical properties of the Gibbs posterior distribution, also known as the Boltzmann distribution, or the exponential weights distribution in e.g. \cite{rigollet2012sparse}, have been studied in e.g. \cite{zhang1999theoretical, zhang2006e, zhang2006information, xu2017information, raginsky2016information} via an information-theoretic viewpoint. Applications of the Gibbs distribution in devising and analyzing training algorithms have been the focus of recent studies such as \cite{chaudhari2016entropy, raginsky2017non, dziugaite2017entropy}. 
Tilted distributions in unsupervised and semi-supervised statistical learning problems has also been studied in \cite{asadi2017compressing} in the context of community detection.
For results on applying MCMC methods to large data sets, see \cite{bardenet2017markov} and references therein.

\paragraph{Notation}
In this paper, all logarithms are in natural base and all information-theoretic measures are in nats. Let $\imath_{P\|Q}$, $D(P\|Q)$ and $D_{\lambda}(P\|Q)$ denote the relative information, the relative entropy, and the R\'{e}nyi divergence of order $\lambda$ between probability measures $P$ and $Q$, and let $D(P_{Y|X}\|Q_{Y|X}|P_X)\triangleq \int D(P_{Y|X=\omega}\|Q_{Y|X=\omega})\mathrm{d}P_X(\omega)$ denote conditional relative entropy (see Appendix \ref{info theory tools} for precise definitions). 
In the framework of supervised statistical learning, $\mathcal{X}$ denotes the instances domain, $\mathcal{Y}$ is the labels domain, $\mathsf{Z}=\mathcal{X}\times \mathcal{Y}$ denotes the examples domain and $\mathcal{H}=\{h_w : w\in \mathcal{W}\}$ is the hypothesis set, where the hypotheses are indexed by an index set $\mathcal{W}$. Let $\ell:\W\times \mathsf{Z}\to \mathbb{R}^+$ be the loss function. A learning algorithm receives the training set $S=(Z_1,Z_2,...,Z_n)$ of $n$ examples with i.i.d.\ random elements drawn from $\mathsf{Z}$ with an unknown distribution $\mu$. Then it picks an element $h_W\in\mathcal{H}$ as the output hypothesis according to a random transformation $P_{W|S}$. For any $w\in\mathcal{W}$, let 
$
L_{\mu}(w)\triangleq \E[\ell(w,Z)] 
$
denote the statistical (or population) risk of hypothesis $h_w$, where $~ Z\sim \mu$. For a given training set $S$, the empirical risk of hypothesis $h_w$ is defined as 
$
L_{S}(w)\triangleq \frac{1}{n}\sum_{i=1}^n \ell(w,Z_i),
$
and the generalization error of hypothesis $h_w$ (dependent on the training set) is defined as
$
\mathrm{gen}(w)\triangleq L_{\mu}(w)-L_S(w).
$
Averaging with respect to the joint distribution $P_{S,W}=\mu^{\otimes n}P_{W|S}$, we denote the expected generalization error  
by
$
\mathrm{gen}(\mu, P_{W|S})\triangleq \E [\gen(W)],
$
and the average statistical risk by  
$
	\risk(\mu, P_{W|S})\triangleq \E[L_{\mu}(W)].
$    
Throughout the paper, $\|A\|_2$ denotes the spectral norm of matrix $A$ and $|b|_2$ denotes the Euclidean norm of vector $b$. Let $\delta_w$ denote the Dirac measure centered at $w$.

\section{Preliminary: The CMI technique}
Chaining, originated from the work of Kolmogorov and Dudley, is a powerful technique in high dimensional probability for bounding the expected suprema of random processes while  taking into account the dependencies between their random variables in a multi-scale manner. Here we emphasize the core idea of the chaining technique: performing refined approximations by using a  telescoping sum, named as \emph{the chaining sum}. If $\{X_t\}_{t\in T}$ is a random process, then for any $t\in T$ one can write
 \begin{equation}
 	X_t=X_{\pi_1(t)}+\left(X_{\pi_2(t)}-X_{\pi_1(t)} \right)+\dots+ \left(X_{\pi_d(t)} -X_{\pi_{d-1}(t)} \right) + \left(X_{t}-X_{\pi_d(t)} \right),\nonumber
 \end{equation}
 where $\pi_1(t), \pi_2(t), \dots, \pi_d(t)$ are finer and finer approximations of the index $t$. Each of the differences $X_{\pi_k(t)}-X_{\pi_{k-1}(t)}$, $k=1,2,\dots, d$, is called a \emph{link} of the chaining sum. Informally speaking, if the approximations $\pi_k(t)$, $k=1,2,\dots, d$, are close enough to each other and $\pi_d(t)$ is close to $t$, then, in many important applications, controlling the expected supremum of each of the links with union bounds and summing them up will give a much tighter bound than bounding the supremum of $X_t$ upfront with a union bound.\footnote{The idea is that the increments may capture more efficiently 
 the dependencies.} 
 For instance, the approximations may be the projections of $t$ on an increasing sequence of partitions
   of $T$. For more information,  
  see \cite{Ramon, Vershynin, talagrand2014upper} and references therein. 
 
The technique of chaining mutual information, recently introduced in \cite{asadi2018chaining}, can be interpreted as an  algorithm-dependent version of the above, extending a result of Fernique \cite{Fernique} by taking into account such dependencies. In brief, \cite{asadi2018chaining} asserts that one can replace the metric entropy in chaining with the mutual information between the input and the discretized output,  
to obtain an upper bound on the expected bias $\E[X_W]$ of an algorithm which selects its output from a random process $\{X_t\}_{t\in T}$.\footnote{The notion of metric entropy is similar to Hartley entropy in the information theory literature. To deal with the effect of noise in communication systems, Hartley entropy was generalized and replaced by mutual information by Shannon (see \cite{verdu1998fifty}).} By writing the chaining sum with random index $W$ and after taking expectations, we obtain:
\begin{equation}\label{CMI expectation sum}
 	\E \left[X_W\right]=\E\left[X_{\pi_1(W)}\right]+\E \left[X_{\pi_2(W)}-X_{\pi_1(W)} \right]+\dots+
 	\E\left[X_{W}-X_{\pi_d(W)} \right].
 \end{equation}  
With this technique, rather than bounding $\E \left[X_W\right]$ with a single mutual information term such as in \cite{russo2016controlling, xu2017information}, one bounds each link $\E \left[X_{\pi_{k}(W)}-X_{\pi_{k-1}(W)} \right]$, $k=1,2,\dots, d$, and then sums them up. 

In this paper, first we note that unlike the classical chaining method in which we require finite size partitions whose cardinalities appear in the bounds,\footnote{Finite partitions is not required in the theory of majorizing measures (generic chaining).}
that requirement is unnecessary for the CMI technique. Therefore one may use a hierarchical sequence of coverings of the index set which includes covers of possibly uncountably infinite size.  
 This fact will be useful for analyzing neural nets with continuous weight values in the next sections. For details, see Appendix \ref{CMI section appendix}.\footnote{Using \cite[Theorem 2]{bu2019tightening}, we also show that for empirical processes, one can replace the mutual information between the whole input set and the discretized output with mutual informations
 between individual examples and the discretized output, to obtain a tighter CMI bound. For details, see Appendix \ref{CMI section appendix}.
} 
 
 The second important contribution is to design the coverings to meet the multilayer structure of neural nets. In the classical chaining and the CMI in \cite{asadi2018chaining}, these are applied on an arbitrary infinite sequence of $2^{-k}$-partitions. In this paper, we take a different approach and use the hierarchical sequences of generated coverings associated with multilevel architectures, as defined in the next section.
 
\section{Multilevel architectures and their generated coverings}\label{generated sequence of partitions section}
Assume that in a statistical learning problem, the hypothesis set $\mathcal{H}$ consists of multilevel functions, i.e., the index set $\W=\W_1\times \cdots \times\W_d$ consists of elements $w\in \W$ representable with $d\geq 2$ components as $w=(\mathbf{W}_1,\dots,\mathbf{W}_d)$. Examples for neural nets can be: 1. When the components are the layers. 2. When the components are stacks of layers plus skip connections, such as in ResNets \cite{he2016deep}. 
For all $1\leq k\leq d$, let $\G_k$ be the exact covering of $\W$ determined by all possible values of the first $k$ components, i.e. any two indices are in the same set if and only their first $k$ components match: 
\begin{equation}
	\mathcal{G}_k\triangleq  \left\{\{\mathbf{W}_1\}\times\cdots\times\{\mathbf{W}_k\}\times \W_{k+1}\times \cdots \times \W_d:  (\mathbf{W}_1,\dots,\mathbf{W}_k)\in \W_1\times \cdots \times \W_k \right\}.\nonumber
\end{equation}
Notice that $\{\G_k\}_{k=1}^d$ is a hierarchical sequence of exact coverings of the index set $\W$, and the projection set of any $w\in \W$ in $\G_k$, i.e., the unique set in $\G_k$ which includes $w$, is determined only by the values of the first $k$ components of $w$. We call $\{\G_k\}_{k=1}^d$ the hierarchical sequence of \emph{generated coverings} of the index set $\mathcal{W}$, and will use the CMI technique on this sequence in the next sections.\footnote{Notice that for a given architecture, one can re-parameterize the components with different permutations of $\{1,2,\dots,d\}$ to give different generated coverings.}
\begin{remark}\normalfont
	The notion of generated coverings of $\W$ is akin in nature to the notion of \emph{generated filtrations} of random processes in probability theory (for a definition, see e.g. \cite[p. 171]{ccinlar2011probability}) and applying the CMI technique on this sequence is akin to the \emph{martingale method}.
\end{remark}
We provide the following simple yet useful example by revisiting Example 1 of \cite{asadi2018chaining}:
\begin{example}\label{Gaussian linear net}
\normalfont 
Consider a canonical Gaussian process $X_t\triangleq \langle t,G^n\rangle, t\in T$ where $n=2$, $G^2=(G_1,G_2)$ has independent standard normal components and $T\triangleq\{t\in \mathbb{R}^2: | t |_2=1\}$. The process $\{X_t\}_{t\in T}$ can also be expressed according to the phase of each point $t\in T$, i.e. the unique number $\phi\in [0,2\pi)$ such that $t=(\sin \phi, \cos \phi)$. Assume that the indices are in the phase form and define the following dyadic sequence of partitions of $T$:
For all integers $k\geq 1$,
\begin{equation}\nonumber
	\pa_k\triangleq \left\{\left[0, \frac{2\pi}{2^{k}}\right), \left[\frac{2\pi}{2^{k}},2\times\frac{2\pi}{2^{k}}\right),..., \left[\left(2^{k}-1\right)\frac{2\pi}{2^{k}}, 2\pi\right)\right\};
\end{equation}
see Figure \ref{Partition_Gaussian}.
\end{example}
\begin{wrapfigure}{l}{0.5\textwidth}
  \begin{center}
    \includegraphics[width=0.5\textwidth]{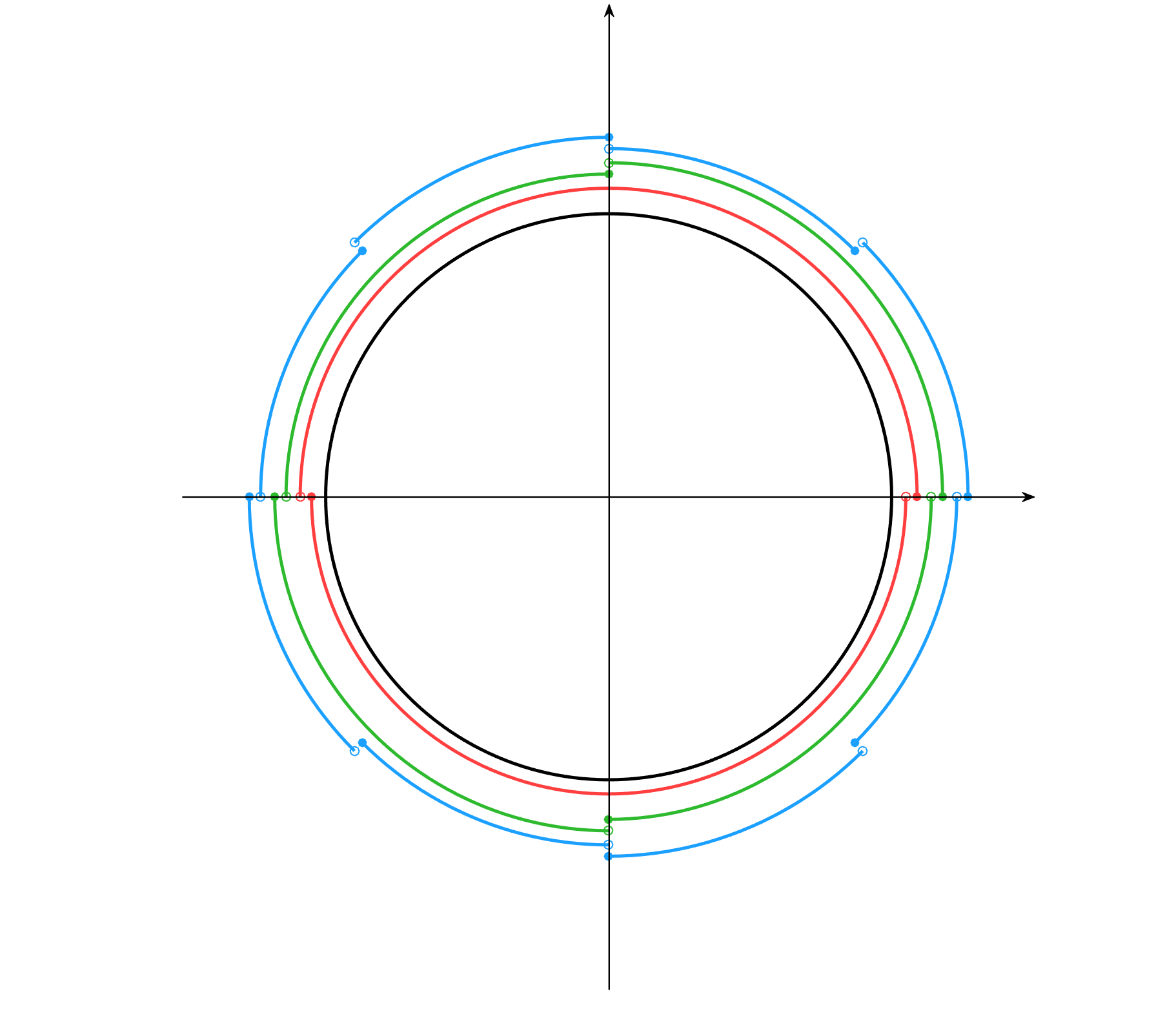}
  \end{center}
  \caption{Dyadic sequence of partitions of $T$}
  \label{Partition_Gaussian}
\end{wrapfigure}
Can $T$ and the sequence $\{\pa_k\}_{k=1}^{\infty}$ be related to the hypothesis set of a multilevel architecture and its generated coverings?
For all integers $i\geq 1$, let
$
	\mathcal{W}_i\triangleq \left\{
	\begin{bmatrix}
		\cos \theta &-\sin \theta\\
		\sin \theta & \cos \theta
	\end{bmatrix}\middle| \theta \in \left\{-\frac{\pi}{2	^i}, \frac{\pi}{2^i} \right\} \right\}.
$
Notice that for each $t=[t_1, t_2]\in T$, one can write 
\begin{align}
	X_t&=
	\begin{bmatrix}
t_1 & t_2	
\end{bmatrix}G^2
\nonumber\\
	 &=\begin{bmatrix}
1 & 0	
\end{bmatrix}
\left(\cdots W_2W_1 \right)G^2,\nonumber
\end{align}
where each $W_i\in \mathcal{W}_i$ is uniquely determined by $t$. Fixing the values of $W_1,\dots, W_k$ and allowing the rest of the matrices to take arbitrary values in their corresponding $\mathcal{W}_i$ gives one of the elements of $\mathcal{P}_k$. Therefore, the sequence of generated coverings associated with the index set of the infinite-depth linear neural net 
\begin{equation}
	f_W(G^2)=\begin{bmatrix}
1 & 0	
\end{bmatrix}
\left(\cdots W_2W_1 \right)G^2\nonumber
\end{equation}
is $\{\pa_k\}_{k=1}^{\infty}$. 
\section{Multilevel regularization}\label{multilevel regularization section}
The purpose of multilevel regularization is to control the diameters of the generated coverings\footnote{The diameter of a covering for a metric space is defined as the supremum of the diameters of its blocks.} and the links of its corresponding chaining sum. Consider a $d$ layer feed-forward neural net with parameters
	$
		w\triangleq (\mathbf{W}_1,\mathbf{W}_2,\dots,\mathbf{W}_d)\in \mathcal{W},
	$ 
	where for all $1\leq k\leq d$, $\mathbf{W}_k\in\mathbb{R}^{\delta_{k}\times\delta_{k-1}}$ is a matrix 
	between hidden layers $k-1$ and $k$.
	 Let $\phi$ denote any non-linearity which is $1$-Lipschitz\footnote{One can readily replace the ReLU activation function with any other $\rho$-Lipschitz activation function which maps the origin to origin. Our bounds in the next section will then depend on $\rho$.} and satisfies $\phi(0)=0$, such as the 
	 entry-wise ReLU activation function, and let $\phi_o$ either be the soft-max function, or the identity function. For a given $R>0$, assume that the instances domain is $\mathcal{X}\triangleq\{x^m\in \mathbb{R}^m: |x^m|_2\leq R \}$. The feed-forward neural net with parameters $w$ is a function $h_w: \mathcal{X} \to \mathbb{R}^{\delta_d}$ defined as 
 $
 	h_w(x^m)\triangleq \phi_o(\mathbf{W}_d(\phi (\cdots \phi(\mathbf{W}_1 (x^m))\cdots))).
 $
For all $1\leq k \leq d$, let $M_k\in\mathbb{R}^{\delta_{k}\times\delta_{k-1}}$ be a fixed matrix 
such that $\|M_k\|_2>0$, and for $\alpha_k>0$, define the following set of matrices:
	\begin{equation}\label{multilevel regularization equation}
		\mathcal{W}_k\triangleq \{\mathbf{W}\in \mathbb{R}^{\delta_{k+1}\times\delta_k}: \|\mathbf{W}-M_k\|_2\leq \alpha_k\|M_k\|_2 \}.
	\end{equation} 
	We assume that the domain of $\mathbf{W}_k$ is restricted to $\mathcal{W}_k$. We are regularizing $\mathbf{W}_k$ with $M_k$ and $\alpha_k$, for all $1\leq k \leq d$, to constrain the 
	links of the chaining sum
	, as we will see in Lemma \ref{chaining link distance}. We name $M_k$ and $\alpha_k$ as the \emph{reference}\footnote{This is similar to the terminology of ``reference matrices" in \cite{bartlett2017spectrally}.} and \emph{radius} of $\mathcal{W}_k$, respectively. A common example used in practice is to let the references be identity matrices, such as for residual nets (see e.g. \cite{hardt2016identity, bartlett2018representing, bartlett2017spectrally}). For instance, for the linear neural net in Example \ref{Gaussian linear net}, we can take $M_k=I_{2\times 2}$ and $\alpha_k=\pi 2^{-k}$, for all $k\geq 1$.  
	
	We define the projection of $w$ on the generated covering $\G_k$ as $(\mathbf{W}_1,\dots,\mathbf{W}_k,M_{k+1},\dots,M_d)$. Let $M\triangleq \prod_{j=1}^{d} \|M_j\|_2$.
	\begin{lemma}\label{chaining link distance} Let $1\leq k\leq d$. Assume that $w_1=(\mathbf{W}_1,\dots,\mathbf{W}_{k-1},\mathbf{W}_k,M_{k+1},\dots,M_d)$ and $w_2=(\mathbf{W}_1,\dots,\mathbf{W}_{k-1},M_k,M_{k+1},\dots,M_d)$. Then, for all $x^m\in \mathcal{X}$,
	\begin{equation}
		|h_{w_1}(x^m)-h_{w_2}(x^m)|_2 \leq \alpha_k\exp\left(\sum_{i=1}^{k-1} \alpha_i\right) M|x^m|_2.\nonumber
	\end{equation}
\end{lemma} 
For a proof, see Appendix \ref{gen bounds section}.

Notice that for any $w\in \W$ and any $x^m\in \mathcal{X}$, if $\phi_o$ is the soft-max function, then $|h_w(x^m)|_2\leq 1$, and if $\phi_o$ is the identity function, then from (\ref{multilevel regularization equation}) and the triangle inequality, we derive $|h_w(x^m)|_2\leq \exp\left(\sum_{i=1}^d \alpha_i\right)MR$.
	Let the loss function $\ell$ be chosen such that there exists\footnote{This assumption is similar to the assumption of Lemma 17.6 in \cite{anthony2009neural}.} $L>0$ for which for any $w_1,w_2\in \W$ and any $z=(x^m,y)\in \Z$ we have
	$|\ell(w_1,z)-\ell(w_2,z)|\leq L|h_{w_1}(x^m)-h_{w_2}(x^m)|_2$ 
	. A commonly used example is the squared $\ell_2$ loss 
	i.e. for the net with parameters $w$ and for any example $z=(x^m,y)\in \mathsf{Z}$, define
	$
		\ell(w,z)\triangleq |h_w(x^m)-y|_2^2
	$.
For classification problems, assume that the labels $y$ are one-hot vectors, otherwise, let $|y|_2\leq 1$.
	Note that for this loss function, if $\phi_o$ is the soft-max function, then one can assume $L=4$, and if $\phi_o$ is the identity function, then one can take $L=2+2\exp\left(\sum_{i=1}^d \alpha_i\right)MR$.
\section{Generalization and excess risk bounds}\label{Gen bounds section main}
	For all $1\leq k \leq d$, let $W_k$ denote a \emph{random} matrix and define 
$
		\gen(W_1,\dots,W_k)\triangleq \gen(\left[W_1,\dots,W_k,M_{k+1},\dots,M_d\right])
	$
	and $\beta_k\triangleq \alpha_k\exp\left(\sum_{i=1}^{k-1} \alpha_i\right).$
We can now state the following multi-scale and algorithm-dependent generalization bound derived from the CMI technique, in which mutual informations between the training set $S$ and the first $k$ layers appear:
\begin{theorem}\label{CMI generalization deep nets theorem}
	Given the assumptions in the previous section, we have 
	\begin{equation}\label{CMI bound neural nets}
		\mathrm{gen}(\mu, P_{W|S})\leq \frac{LMR\sqrt{2}}{\sqrt{n}}\sum_{k=1}^{d}\beta_k\sqrt{I(S;W_1,\dots,W_k)}.
	\end{equation}
\end{theorem}  
{\it Proof outline.} According to (\ref{CMI expectation sum}), one can write the chaining sum with respect to the sequence of generated coverings as
\begin{align}
		\gen(\mu, P_{W|S})= \E[\gen(W)]&=\E[\gen(W_1)]+\E[\gen(W_1,W_2)-\gen(W_1)]+\dots \nonumber\\
		&\quad +\E[\gen(W)-\gen(W_1,\dots,W_{d-1})],\nonumber
	\end{align}
	while, based on Lemma \ref{chaining link distance}, observe that for all $1\leq k\leq d$, 	
	\begin{equation}
		\E[\gen(W_1,\dots,W_k)-\gen(W_1,\dots,W_{k-1})]\leq \frac{LMR\sqrt{2}\beta_k}{\sqrt{n}} \sqrt{I(S;W_1,\dots,W_k)}.\nonumber
	\end{equation}
For a complete proof, see Appendix \ref{gen bounds section}.

Notice that we can rewrite (\ref{CMI bound neural nets}) as 
\begin{equation}\label{CMI gen bound 2}
	\risk\left(\mu, P_{W|S}\right)=\E[L_{\mu}(W)]\leq \E[L_S(W)]+\frac{C}{\sqrt{n}}\sum_{k=1}^d \beta_k\sqrt{I(S;W_1,\dots,W_k)},
\end{equation}
where $C\triangleq LMR\sqrt{2}$. The goal in statistical learning is to find an algorithm $P_{W|S}$ which minimizes $\risk\left(\mu, P_{W|S}\right)= \E[L_{\mu}(W)].$
To that end, we derive an upper bound on $\E[L_{\mu}(W)]$ from inequality (\ref{CMI gen bound 2}) whose minimization over $P_{W|S}$ is algorithmically feasible.
If for each $k=1,2,\dots, d$, we define $Q^{(k)}_{W_1\dots W_k}$ to be a fixed distribution on $\mathcal{W}_1\times \cdots \times \mathcal{W}_k$ that does not depend on the training set $S$, which we name as \emph{prior distribution},\footnote{Similar to the terminology in PAC-Bayes theory (see e.g. \cite{catoni2007pac}). 
} 
then from (\ref{CMI gen bound 2}) we deduce 
\begin{align}
	\risk\left(\mu, P_{W|S}\right)&\leq \E[L_S(W)]+\frac{C}{\sqrt{n}}\sum_{k=1}^d \beta_k\left(\gamma_k I(S;W_1,\dots,W_k)+\frac{1}{4\gamma_k}\right)\label{square root upper bound}\\
	     					&\leq \E[L_S(W)]+\frac{C}{\sqrt{n}}\sum_{k=1}^d \beta_k\left(\gamma_k D\left(P_{W_1\dots W_k|S}\middle\|Q^{(k)}_{W_1\dots W_k}\middle|P_S\right)+\frac{1}{4\gamma_k}\right),\label{minimizing objective}
\end{align}
where (\ref{square root upper bound}) follows from the inequality 
$
		\sqrt{x}\leq cx+\frac{1}{4c}  
	$ for all $x,c>0$,
which is upper bounding the concave function $\sqrt{x}$ with a tangent line, and (\ref{minimizing objective}) follows from the crucial difference decomposition of mutual information: $I(X;Y)=D(P_{Y|X}\|Q_Y|P_X)-D(P_Y\|Q_Y)$; see Lemma \ref{Golden formula for MI} in Appendix \ref{info theory tools}. 
Given fixed parameters $\gamma_k$, $k=1,2,\dots,d$, and for any fixed $n$, let $P^{\star}_{W|S}$ be the conditional distribution which minimizes the right side of (\ref{minimizing objective}), i.e. 
\begin{equation}\label{minimizing objective 3}
	P^{\star}_{W|S}\triangleq \argmin_{P_{W|S}} \left\{ \E[L_S(W)]+\frac{C}{\sqrt{n}}\sum_{k=1}^d \beta_k\gamma_k D\left(P_{W_1\dots W_k|S}\middle\|Q^{(k)}_{W_1\dots W_k}\middle|P_S\right)\right\}.
\end{equation}
Note that we made the expression in (\ref{minimizing objective 3}) linear in $P_S$. This, in turn, implies that the algorithm $P^{\star}_{W|S}$ does not depend on the unknown input distribution $\mu$ (recall that $P_S=\mu^{\otimes n}$), which is a desired property of $P^{\star}_{W|S}$. For discrete $\W$, the algorithm $P_{W|S}^{\star}$ achieves the following excess risk bound: 
 \begin{theorem}\label{CMI excess risk} Assume that $\W$ is a discrete set and for a given input distribution $\mu$, let $\widehat{w}(\mu)$ denote the index of a hypothesis which achieves the minimum statistical risk among $\W$. Then
	\begin{equation}
		\risk\left(\mu, P_{W|S}^{\star}\right)\leq \inf_{w\in\W}L_{\mu}(w)+\frac{C}{\sqrt{n}}\sum_{k=1}^d \beta_k\left(\gamma_k D\left(\delta_{\widehat{w}_1 \dots \widehat{w}_k}\middle\|Q^{(k)}_{W_1\dots W_k}\right)+\frac{1}{4\gamma_k}\right). \label{minimizing objective 2}
\end{equation} 
\end{theorem}
Note that, for all $1\leq k\leq d$, the relative entropies in Theorem \ref{CMI excess risk} are computed as
\begin{equation}
    D\left(\delta_{\widehat{w}_1 \dots \widehat{w}_k}\middle\|Q^{(k)}_{W_1\dots W_k}\right)=\log \frac{1}{Q^{(k)}_{W_1\dots W_k}(\widehat{w}_1, \dots ,\widehat{w}_k)}. \nonumber
\end{equation}
For a proof of Theorem \ref{CMI excess risk}, a high-probability version, and a result for the case when $\W$ is not discrete, see Appendix \ref{gen bounds section}. A case of special and practical interest is when the prior distributions are consistent, i.e., when there exists a single distribution $Q_{W_1\dots W_d}$ such that $Q^{(k)}_{W_1\dots W_k}=Q_{W_1\dots W_k}$ for all $1\leq k\leq d$. In this case, both (\ref{minimizing objective 3}) and (\ref{minimizing objective 2}) can be expressed with the following new divergence: 
\begin{definition}[Multilevel relative entropy]\label{multilevel relative entropy definition} For probability measures $P_{X_1\dots X_n}$ and $Q_{X_1\dots X_n}$, and a vector $\mathbf{a}=(a_1,\dots, a_n)\in \mathbb{R}_+^{n}$, define the \emph{multilevel relative entropy} as
\begin{equation}
	D_{(\mathbf{a})}\left(P_{X_1\dots X_n}\|Q_{X_1\dots X_n}\right)\triangleq\sum_{i=1}^n a_iD\left(P_{X_1\dots X_i}\|Q_{X_1\dots X_i}\right).
\end{equation} 
\end{definition}
The prior distributions $Q^{(k)}_{W_1\dots W_k}$ may be given by Gaussian matrices truncated on  bounded-norm sets. 

It is shown in \cite{xu2017information} (with related results in \cite{zhang2006information, rigollet2012sparse})  that the Gibbs posterior distribution $P^{\gamma,Q}_{W|S}\propto e^{-\gamma L_s(w)}Q$, as defined precisely 
in Definition \ref{Gibbs distribution definition} 
in Appendix \ref{Gibbs algorithms section}, is the unique solution to 
\begin{equation}\nonumber
	\argmin_{P_{W|S}}\left\{\E[L_S(W)]+\frac{1}{\gamma}D(P_{W|S}\|Q|P_S)\right\},
\end{equation}
where $\gamma$ is called the \emph{inverse temperature}.
Thus, based on (\ref{minimizing objective 3}), the desired distribution $P^{\star}_{W|S}$ is a multi-scale generalization of the Gibbs distribution. In the next section, we obtain the functional form of $P^{\star}_{W|S}$.
Inspired from the terminology for the Gibbs distribution, we call the vector of coefficients $\left(\frac{C\beta_1\gamma_1}{\sqrt{n}},\dots,\frac{C\beta_d\gamma_d}{\sqrt{n}}\right) $ in (\ref{minimizing objective 3}) the \emph{temperature vector} of $P^{\star}_{W|S}$.
Note that for minimizing the excess risk bound (\ref{minimizing objective 2}), the optimal value for $\gamma_k$, for all $1\leq k\leq d$, is
\begin{equation*}
	\gamma_k^{\star}=\frac{1}{2\sqrt{D\left(\delta_{\widehat{w}_1 \dots \widehat{w}_k}\middle\|Q^{(k)}_{W_1\dots W_k}\right)}}.
\end{equation*}

Furthermore, as a byproduct of the above analysis, we give new excess risk bounds for the Gibbs distribution in Propositions \ref{Gibbs excess _ statistical learning} and \ref{Gibbs excess risk uncountable} 
in Appendix \ref{Gibbs algorithms section} (a related result has recently been obtained in \cite{kuzborskij2019distribution}, though using stability arguments). These results generalize Corollaries 2 and 3 in \cite{xu2017information} to arbitrary subgaussian losses, and unlike their proof which is based on stability arguments of \cite{raginsky2016information}, merely uses the mutual information bound \cite{russo2016controlling,  xu2017information}.

\section{The Marginalize-Tilt (MT) algorithm}\label{MT section}
The optimization problem (\ref{minimizing objective 3}), which was derived by \emph{chaining} mutual information, can be solved via the \emph{chain rule} of relative entropy, and based on a key property of conditional relative entropy (Lemma \ref{conditional relative entropy positive} in Appendix \ref{MT procedure proof section}), can be shown to have a unique solution.
Note that
if we know the solution to the following more general relative entropy sum minimization:
\begin{equation}\label{multilevel relative entropy sum minimization}
	\argmin_{P_{X_1\dots X_d}} \left\{a_1D\left(P_{X_1}\middle\|R^{(1)}_{X_1}\right)+a_2D\left(P_{X_1X_2}\middle\|R^{(2)}_{X_1X_2}\right)+\dots+a_dD\left(P_{X_1\dots X_d}\middle\|R^{(d)}_{X_1\dots X_d}\right)\right\},
\end{equation}
where $a_i>0$ and distributions $R^{(i)}_{X_1\dots X_i}$ are given for all $1\leq i \leq d$, then we can use that to solve for $P^{\star}_{W|S=s}$ in (\ref{minimizing objective 3}) for any $s\in \Z^n$, by assuming the following: $X_i\triangleq W_i$ and 
$a_i \gets \frac{C\beta_i\gamma_i}{\sqrt{n}} $ for all $1\leq i\leq d$, $R^{(i)}\gets Q^{(i)}$ for all $1\leq i\leq d-1$, and
\begin{equation}\nonumber
	R^{(d)}(\mathrm{d}x)\gets \frac{e^{-\frac{\sqrt{n}}{C\beta_d\gamma_d} L_{s}(x)}Q^{(d)}(\mathrm{d}x) }{\E\left[e^{-\frac{\sqrt{n}}{C\beta_d\gamma_d} L_{s}(\widetilde{X})}\right]},\quad \widetilde{X}\sim Q^{(d)},
\end{equation} 
where we combined the expected empirical risk with the last relative entropy in (\ref{minimizing objective 3}) and ignored the resulting term which does not depend of $P_{X_1\dots X_n}$ (such combination is similarly performed in \cite[Section IV]{zhang2006information} for proving the optimality of the Gibbs distribution).
The solution to (\ref{multilevel relative entropy sum minimization}), denoted as $P^{\star}_{X_1\dots X_d}$, is the output of Algorithm \ref{MT}. If $P$ and $Q$ are distributions on a set $\mathcal{A}$, then let the relative information $\nonumber
	\imath_{P\|Q}(a)=\log \frac{\mathrm{d}P}{\mathrm{d}Q}(a)
$ denote the logarithm of the Radon--Nikodym derivative of $P$ with respect to $Q$ for all $a\in \mathcal{A}$.  The algorithm uses the following:
\begin{definition}[Tilted distribution\footnote{The tilted distribution is known as the \emph{generalized escort distribution} in the statistical physics and the statistics literatures (see e.g. \cite{bercher2012simple}).}]\label{Tilted distribution definition} Given distributions $P$ and $Q$, let $R$ be a dominating measure such that $R \gg P$ and $R \gg Q$. The tilted distribution $(P,Q)_{\lambda}\ll R$ for $\lambda\in [0,1]$ is defined with 
\begin{equation}
	\imath_{(P,Q)_{\lambda}\|R}(a)=\lambda \imath_{P\|R}(a)+(1-\lambda)\imath_{Q\|R}(a)+(1-\lambda)D_{\lambda}(P\|Q),\nonumber
\end{equation}
for all $a\in\mathcal{A}$. If $P \perp Q$, then $(P,Q)_{\lambda}$ is not defined for $\lambda\in (0,1)$.
\end{definition}
\begin{remark}\normalfont
	In the special case that $P$ and $Q$ are distributions on a discrete set $\mathcal{A}$, for all $a\in \mathcal{A}$, we have
\begin{equation*}
	(P,Q)_{\lambda}(a)= \frac{P^{\lambda}(a)Q^{1-\lambda}(a)}{\sum_{x\in \mathcal{A}}P^{\lambda}(x)Q^{1-\lambda}(x)}.
\end{equation*}
 In the case that $P$ and $Q$ are distributions of real-valued absolutely continuous random variables with probability density functions $f_0$ and $f_1$, the tilted random variable has probability density function
 \begin{equation*}
 	f_{\lambda}(x)=\frac{e^{\lambda \log f_0(x)+(1-\lambda)\log f_1(x)}}{\int_{-\infty}^{\infty}e^{\lambda \log f_0(t)+(1-\lambda)\log f_1(t)}\mathrm{d}t}.
 \end{equation*}
 Notice that $(P,Q)_{\lambda}$ traverses between $Q$ and $P$ as $\lambda$ traverses between $0$ and $1$. 
\end{remark}
The following shows the useful role of tilted distributions in linearly combining relative entropies. For a proof, see \cite[Theorem 30]{van2014renyi}.
\begin{lemma}\label{Tilted Renyi lemma main}
Let $\lambda\in [0,1]$. For any $P \ll Q$ and $P\ll R$,
\begin{equation}
	\lambda D(P\|Q)+(1-\lambda)D(P\|R)=D\left(P\|(Q,R)_{\lambda}\right)+(1-\lambda)D_{\lambda}(Q\|R).\nonumber
\end{equation}	
\end{lemma}
\begin{theorem}\label{MT solution theorem}
	The output of Algorithm \ref{MT} is the unique solution to (\ref{multilevel relative entropy sum minimization}).
\end{theorem}
{\it Proof outline.}
Algorithm \ref{MT} solves for $P^{\star}_{X_1\dots X_d}$ in a backwards manner: Starting from the last term in (\ref{multilevel relative entropy sum minimization}), the algorithm uses the chain rule of relative entropy (see Lemma \ref{chain rule lemma} in Appendix \ref{info theory tools}) to decompose it into two terms; a relative entropy and a conditional relative entropy:
\begin{equation}
    a_d D\left(P_{X_1\dots X_{d-1}}\middle\|R^{(d)}_{X_1\dots X_{d-1}}\right)+a_d D\left(P_{X_d|X_1\dots X_{d-1}}\middle\|R^{(d)}_{X_d|X_1\dots X_{d-1}}\middle|P_{X_1\dots X_{d-1}}\right).\nonumber
\end{equation}
Then, based on Lemma \ref{Tilted Renyi lemma main}, it linearly combines the relative entropy with the previous term in (\ref{multilevel relative entropy sum minimization}) using the corresponding tilted distribution. The algorithm iterates these two steps to reduce solving (\ref{multilevel relative entropy sum minimization}) to a simple problem: minimizing a sum of conditional relative entropies which all can be set equal to zero, \emph{simultaneously}. This is accomplished with $P^{\star}_{X_1\dots X_d}$ given in line \ref{Output of MT}. For a complete proof, see Appendix \ref{MT procedure proof section}. The proof 
also implies that the minimum value of the expression in (\ref{multilevel relative entropy sum minimization}) is a summation of R\'{e}nyi divergences between functions of distributions 
$R^{(i)}_{X_1\dots X_i}$, $1\leq i\leq d$. 
\begin{algorithm}
    \caption{Marginalize-Tilt (MT)}
    \label{MT}
    \begin{algorithmic}[1]      
    \Require Distributions $R^{(i)}_{X_1\dots X_i}$ and coefficients $a_i$, for all $1\leq i \leq d$.
	\Ensure Solution $P^{\star}_{X_1\dots X_d}$ to the minimization problem (\ref{multilevel relative entropy sum minimization}).
                 \State $S^{(d)}_{X_1\dots X_d}\gets R^{(d)}_{X_1\dots X_d}$
                 \For{$k=d-1 \texttt{ to } 1$}
        	 	 \State $M_{X_1\dots X_{k}}\gets S^{(k+1)}_{X_1\dots X_{k}}$ \Comment{The marginalization step}
        	 	 \State $S^{(k)}_{X_1\dots X_{k}}\gets \left(R_{X_1\dots X_{k}}^{(k)},M_{X_1\dots X_{k}}\right)_{\frac{a_{k}}{a_{k}+\dots +a_d}}$ \Comment{The tilting step}
     			 \EndFor
            \State \textbf{return} $P^{\star}_{X_1\dots X_d}= S^{(1)}_{X_1}S^{(2)}_{X_2|X_1}\dots S^{(d)}_{X_d|X_1\dots X_{d-1}}$\Comment{The unique solution to (\ref{multilevel relative entropy sum minimization})}\label{Output of MT}
    \end{algorithmic}
\end{algorithm}

\section{Multilevel training}\label{multilevel training section}
By using the MT algorithm to solve (\ref{minimizing objective 3}), we obtain the ``twisted distribution'' $P^{\star}_{W|S=s}$ for all $s\in \Z^n$. We now seek an efficient implementation of the MT algorithm. We define the multilevel training as simulating $P^{\star}_{W|S=s}$, given the training set $S=s$. 
For a two layer net, we implement this with Algorithm \ref{multilevel Metropolis}.
Let $f(w_1,w_2)\triangleq e^{-L_s(w_1,w_2)}$, where $w_1$ and $w_2$ are the matrices of the first and second layer, respectively.\footnote{In this section, we are denoting matrices with lower case for clarity.} In the important case of having consistent product priors, 
i.e., when we can write $Q^{(1)}(w_1)=\tilde{Q}^{(1)}(w_1)$ and $Q^{(2)}(w_1,w_2)=\tilde{Q}^{(1)}(w_1)\tilde{Q}^{(2)}(w_2)$, assuming temperature vector $(a_1,a_2)$, distribution $P^{\star}_{W|S=s}(w_1,w_2)$ is equal to: 
\begin{equation}\label{twisted distribution simplified}
	\frac{\left(\bigintsss_{v_2}f(w_1,v_2)^{\frac{1}{a_2}} \tilde{Q}^{(2)}(v_2)\mathrm{d}v_2 \right)^{\frac{a_2}{a_1+a_2}}\tilde{Q}^{(1)}(w_1)}{\bigintsss_{v_1}\left(\bigintsss_{v_2}f(v_1,v_2)^{\frac{1}{a_2}}\tilde{Q}^{(2)}(v_2)\mathrm{d}v_2 \right)^{\frac{a_2}{a_1+a_2}}\tilde{Q}^{(1)}(v_1)\mathrm{d}v_1}\times \frac{f(w_1,w_2)^{\frac{1}{a_2}}\tilde{Q}^{(2)}(w_2)}{\bigintsss_{v_2}f(w_1,v_2)^{\frac{1}{a_2}}\tilde{Q}^{(2)}(v_2)\mathrm{d}v_2},
\end{equation}
see Appendix \ref{multilevel Metropolis appendix} for more details.
\begin{algorithm}
    \caption{Multilevel Metropolis}
    \label{multilevel Metropolis}
    \begin{algorithmic}[1] 
     \Require Distributions $\tilde{Q}^{(1)}$ and $\tilde{Q}^{(2)}$, temperature vector $\mathbf{a}=(a_1,a_2)$, proposals $q_1$ and $q_2$, inner level running time $T'$, and initializations $(w_1^{(1)},w_2^{(0)})$. 
	\Ensure A sequence $(w_1^{(t)},w_2^{(t)})_{t=1}^T$ drawn from $P^{\star}_{W|S=s}$ in (\ref{twisted distribution simplified}).
	\For{$t=1 \texttt{ to } T$}
	\State $\widehat{w}_1\sim q_1\left(w_1^{(t)}\right)$ \Comment{Symmetric proposal}
	\State Initialize $v_2^{(0)}\gets w_2^{(t-1)}$, generate sequence $\{v_2^{(i)}\}_{i=0}^{T'}$ drawn from distribution $\frac{f\left(w_1^{(t)},v_2\right)^{\frac{1}{a_2}} \tilde{Q}^{(2)}(v_2)}{\int_{v_2}f\left(w_1^{(t)},v_2\right)^{\frac{1}{a_2}} \tilde{Q}^{(2)}(v_2)\mathrm{d}v_2}$, and let $w_2^{(t)}\gets v_2^{(T')}$. \Comment{Inner level Metropolis algorithm}\label{inner level}
	\State Approximate 
$\frac{\int_{v_2}f\left(\widehat{w}_1,v_2\right)^{\frac{1}{a_2}} \tilde{Q}^{(2)}(v_2)\mathrm{d}v_2}{\int_{v_2}f\left(w_1^{(t)},v_2\right)^{\frac{1}{a_2}} \tilde{Q}^{(2)}(v_2)\mathrm{d}v_2}\approx \frac{1}{T'}\sum_{i=1}^{T'}\left(\frac{f\left(\widehat{w}_1,v_2^{(i)}\right)}{f\left(w^{(t)}_1,v_2^{(i)}\right)}\right)^{\frac{1}{a_2}}\triangleq A$. \label{approximation step} 
	\State $\alpha\gets A^{\frac{a_2}{a_1+a_2}}\times\frac{\tilde{Q}^{(1)}\left(\widehat{w}_1\right)}{\tilde{Q}^{(1)}\left(w_1^{(t)}\right)}$\Comment{Acceptance ratio}
	\State $U\sim \mathrm{Unif}[0,1]$   \Comment{Uniform distribution}
	\If{ $U\leq \alpha$ } \State{ $w_1^{(t+1)}\gets \widehat{w}_1$} \Comment{Accept proposal} \Else { $w_1^{(t+1)}\gets w_1^{(t)}$} \Comment{Reject proposal and keep current state}
	\EndIf 
	\EndFor
   \end{algorithmic}
\end{algorithm}
Algorithm \ref{multilevel Metropolis} consists of two Metropolis algorithms, one in an outer level to sample $\{w_1^{(t)}\}_{t=1}^T$ with distribution as the first fraction in (\ref{twisted distribution simplified}), and the other in the inner level at line \ref{inner level} to sample $\{w_2^{(i)}\}_{i=1}^{T'}$ given $w_1^{(t)}$ with conditional distribution equal to second fraction in (\ref{twisted distribution simplified}). Line \ref{approximation step}, which can be run concurrently with line \ref{inner level},   
shows how the inner level sampling is used in the outer level algorithm: 
Note that to compute the acceptance ratio of the outer level algorithm, we can write 
\begin{align}
	\frac{\bigintsss_{v_2}f\left(\widehat{w}_1,v_2\right)^{\frac{1}{a_2}} \tilde{Q}^{(2)}(v_2)\mathrm{d}v_2}{\bigintss_{v_2}f\left(w_1^{(t)},v_2\right)^{\frac{1}{a_2}} \tilde{Q}^{(2)}(v_2)\mathrm{d}v_2}&=\frac{\bigintss_{v_2}\left(\frac{f\left(\widehat{w}_1,v_2\right)}{f\left(w_1^{(t)},v_2\right)}\right)^{\frac{1}{a_2}} f\left(w_1^{(t)},v_2\right)^{\frac{1}{a_2}}\tilde{Q}^{(2)}(v_2)\mathrm{d}v_2}{\bigintss_{v_2}f\left(w_1^{(t)},v_2\right)^{\frac{1}{a_2}} \tilde{Q}^{(2)}(v_2)\mathrm{d}v_2},\nonumber\\
	&=\E\left[\left(\frac{f\left(\widehat{w}_1,V_2\right)}{f\left(w_1^{(t)},V_2\right)}\right)^{\frac{1}{a_2}} \right],\nonumber
\end{align}
where for any fixed $w_1^{(t)}$,
\begin{equation}
    V_2\sim \frac{ f\left(w_1^{(t)},v_2\right)^{\frac{1}{a_2}}\tilde{Q}^{(2)}(v_2)\mathrm{d}v_2}{\bigintss_{v_2}f\left(w_1^{(t)},v_2\right)^{\frac{1}{a_2}} \tilde{Q}^{(2)}(v_2)\mathrm{d}v_2}.\nonumber
\end{equation}
This justifies the Monte Carlo approximation in line \ref{approximation step}. 
 The initialization at line \ref{inner level} is chosen to let the inner level algorithm mix faster along with the mixing of the outer level algorithm.  Algorithm \ref{multilevel Metropolis} reduces the dimensionality of the proposal distributions, which is a desired property, compared to simulating the Gibbs distribution when $w_1$ and $w_2$ are sampled jointly. For more details and explanations about Algorithm \ref{multilevel Metropolis}, see Appendix \ref{multilevel Metropolis appendix}. 
\begin{example}\label{MNIST}\normalfont
We tested a basic implementation of 
Algorithm \ref{multilevel Metropolis}
with random walk Gaussian proposals on the MNIST data set (as a proof of concept). 
We used a two-layer net of size $784-100-10$ with ReLU activation function for the hidden layer, soft-max activation function for the output layer, and with squared $\ell_2$ loss function.
We let $\mathbf{a}=(2\times 10^{-6},10^{-6})$, $T'=10$ and ran the outer level algorithm for $T=40000$ iterations; see Figures \ref{fig:2figsA} and \ref{fig:2figsB}. This number of iterations is large, in part due to the fact that we did not use any tricks to speed up the algorithm, such as tuning the proposals variances during the burn-in period, or lowering the temperatures gradually as in simulated annealing. 
For more details about this experiment, see Appendix \ref{Experiments section}. The code is available at \url{https://github.com/ARAsadi/Multilevel-Metropolis}. 

\begin{figure}[h]
\centering
\parbox{6.7cm}{
\includegraphics[width=6.7cm]{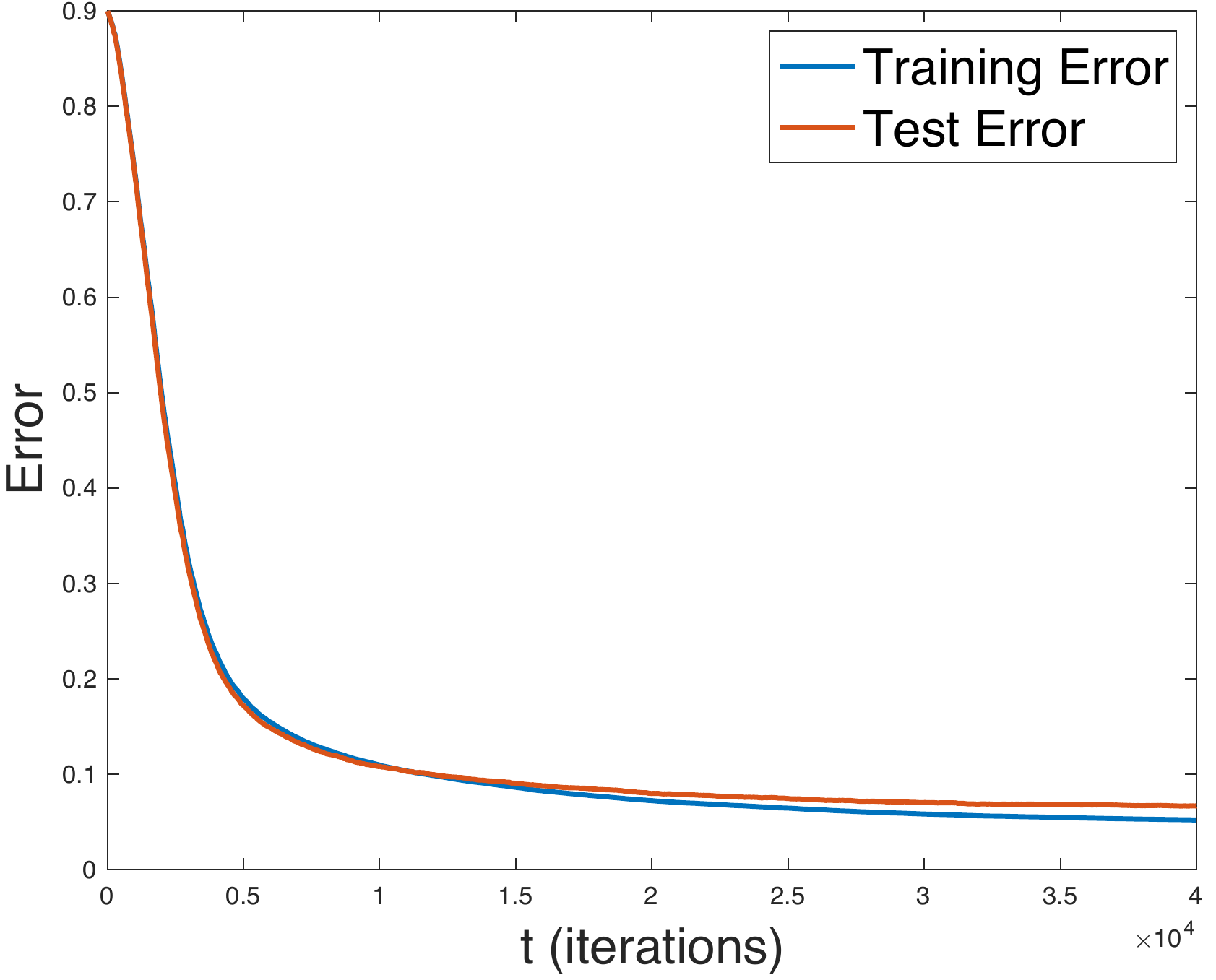}
\caption{Example 2}
\label{fig:2figsA}}
\quad
\begin{minipage}{6.7cm}
\includegraphics[width=6.7cm]{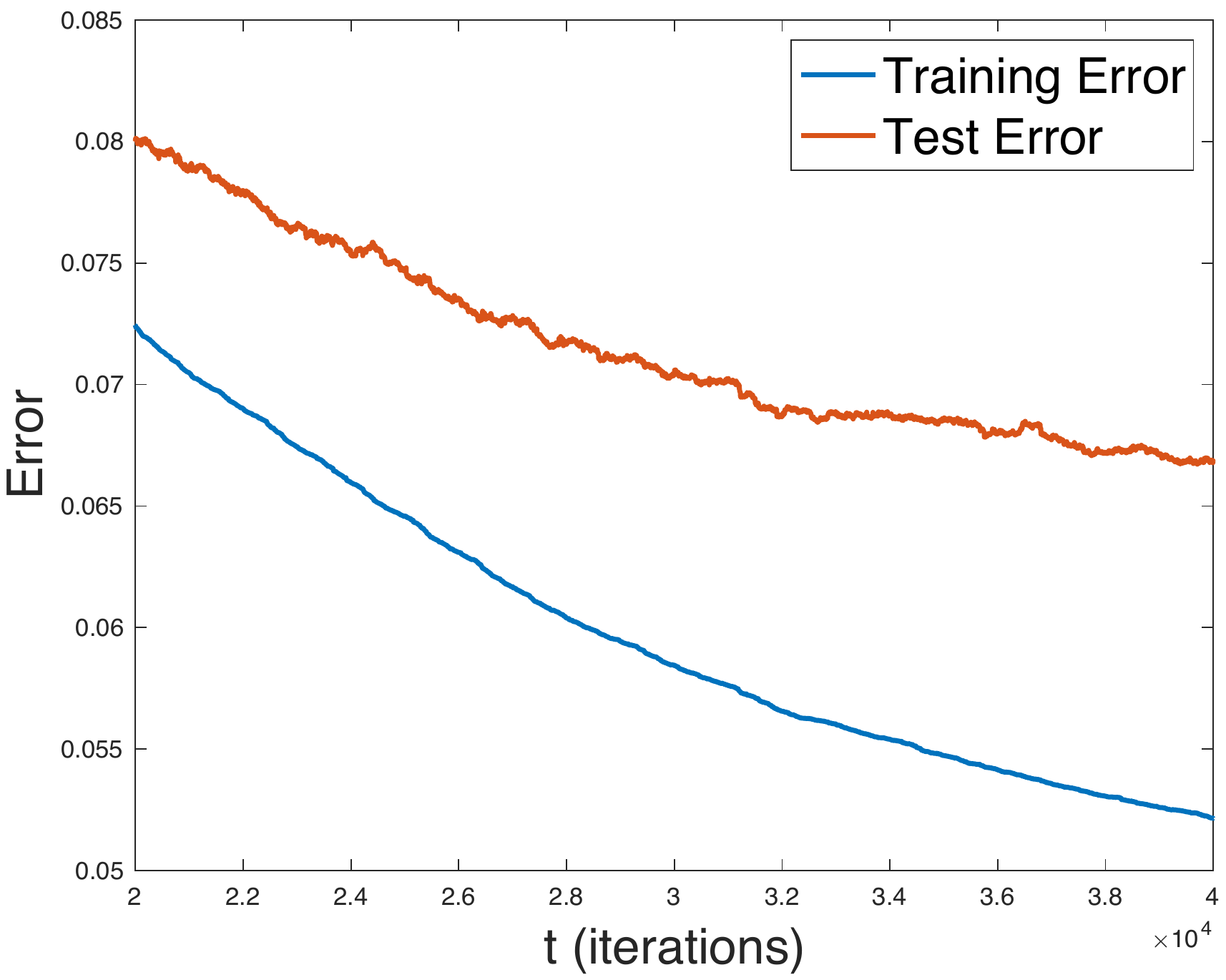}
\caption{Example 2}
\label{fig:2figsB}
\end{minipage}
\end{figure}
\end{example} 
Tuning the temperature parameter
for simulating the Gibbs distribution
is usually done with cross-validation \cite{catoni2007pac, 
guedj2019primer}. We leave for future work the problem of tuning the temperature vector for achieving low test error while having 
low mixing time.  
To simulate $P^{\star}_{W|S}$ for more than two layers, similar to line \ref{approximation step} of Algorithm \ref{multilevel Metropolis}, one can compute Monte Carlo approximations to the acceptance ratio of each layer, based on the samples from the next layers and the inner level algorithms.
Various ideas could be used to decrease the running time of simulating the twisted distribution $P^{\star}_{W|S}$. In particular, one may use  gradients as in Hamiltonian Monte Carlo \cite{neal1992bayesian, chen2014stochastic} and stochastic gradient Langevin dynamics \cite{welling2011bayesian}, divide the training set into mini-batches with divide-and-conquer approaches, use sub-sampling methods \cite{bardenet2017markov}, or simulate a variational Bayes approximation to the twisted distribution (see \cite{alquier2016properties} for approximating the Gibbs distribution). We are currently invesitgating these directions. 

\begin{remark}\normalfont
As a side result, in Appendix \ref{Gibbs average predictors}, we show how to alternatively achieve the excess risk bound of Theorem \ref{CMI excess risk} with an \emph{average predictor} for the special case of binary classification with $\ell_1$ loss, based on an idea of \cite{cesa1999prediction}.  
\end{remark}

\section{Acknowledgement}
We are grateful to Ramon van Handel for his generous time and for the many discussions on chaining.
\bibliographystyle{unsrt}
\bibliography{Biblio}
\newpage
\appendix
\section{Information-theoretic tools}\label{info theory tools}
\begin{definition}[Relative information] Given probability measures $P$ and $Q$ defined on a measurable space $(\mathcal{A},\mathscr{F})$, such that $P\ll Q$, the relative information between $P$ and $Q$ in $a\in \mathcal{A}$ is the logarithm of the Radon--Nikodym derivative of $P$ with respect to $Q$:
\begin{equation}\nonumber
	\imath_{P\|Q}(a)=\log \frac{\mathrm{d}P}{\mathrm{d}Q}(a).
\end{equation}
\end{definition}
\begin{definition}[Relative entropy] The relative entropy between distributions $P$ and $Q$ defined on the same measurable space $(\mathcal{A}, \mathscr{F})$, if $P\ll Q$ is 
\begin{equation}\nonumber
	D(P\|Q)=\mathbb{E}[\imath_{P\|Q}(X)], \qquad X\sim P,
\end{equation}
otherwise, we define $D(P\|Q)=\infty$.
\end{definition}
\begin{definition}[Conditional relative entropy] The conditional relative entropy is defined as 
\begin{align}
	D(P_{Y|X}\|Q_{Y|X}|P_X)&=\int D(P_{Y|X=\omega}\|Q_{Y|X=\omega})\mathrm{d}P_X(\omega)\nonumber\\
						   &=\E[D(P_{Y|X}(\cdot|X)\|Q_{Y|X}(\cdot|X))], \quad X\sim P_X. \nonumber
\end{align}
\end{definition}

The following lemma is known as the chain rule of relative entropy. For a proof of this property of relative entropy, see e.g. \cite[Theorem 2.5.3]{Cover}:
\begin{lemma}[Chain rule of relative entropy]\label{chain rule lemma} We have
\begin{equation}
	D(P_{XY}\|Q_{XY})=D(P_X\|Q_X)+D(P_{Y|X}\|Q_{Y|X}|P_X). \nonumber
\end{equation}
More generally,
\begin{equation}
	D\left(P_{X_1\dots X_n}\|Q_{X_1\dots X_n}\right)=\sum_{i=1}^n D\left(P_{X_i|X_1\dots X_{i-1}}\|Q_{X_i|X_1\dots X_{i-1}}|P_{X_1\dots X_{i-1}}\right). \nonumber
\end{equation}
\end{lemma}

The following is a well-known property of mutual information:
\begin{lemma}[Difference decomposition of mutual information]\label{Golden formula for MI}
	For any $Q_Y$ such that $D(P_Y\|Q_Y)<\infty$, we have 
	\begin{equation}
		I(X;Y)=D(P_{Y|X}\|Q_Y|P_X)-D(P_Y\|Q_Y). \nonumber
	\end{equation}
\end{lemma}

We give the following general definition of R\'{e}nyi divergence from \cite{verdu2015alpha}:
\begin{definition}[R\'{e}nyi divergence]
Given distributions $P$ and $Q$ defined on the same probability space, let probability measure $R$ be such that $P \ll R$ and $Q \ll R$, and let $Z \sim R$. Then, the R\'{e}nyi divergence of order $\alpha \in (0, 1) \cup (1,\infty)$ between $P$ and $Q$ is defined as
\begin{equation}
    D_{\alpha}(P\|Q)=\frac{1}{\alpha-1}\log\E \left[\exp\left(\alpha \imath_{P\|R}(Z)+(1-\alpha)\imath_{Q\|R}(Z)\right)\right]. \nonumber
\end{equation}
Due to its limiting behaviour, for $\alpha=1$ we define $D_1(P\|Q)=D(P\|Q)$. 
\end{definition}
For instance, for discrete distributions $P$ and $Q$ defined on a set $\mathcal{A}$ and for any $\alpha \in (0, 1) \cup (1,\infty)$, we have
\begin{equation}
    D_{\alpha}(P\|Q)=\frac{1}{\alpha-1}\log\left(\sum_{a\in\mathcal{A}}P^{\alpha}(a)Q^{1-\alpha}(a) \right).\nonumber
\end{equation}

\section{Chaining mutual information}\label{CMI section appendix}
In this section, we strengthen the results of \cite{asadi2018chaining}. First we give the necessary definitions:
\begin{definition}[Subgaussian process] \label{SubgaussianProcessDefinition} The random process $\{X_t\}_{t\in T}$ on the metric space $(T,d)$ is called \emph{subgaussian} if $\E[X_t]=0$ for all $t\in T$ and 
\begin{equation}\nonumber
\E[e^{\lambda(X_t-X_s)}]\leq e^{\frac12  \lambda^2 d^2(t,s)} \textrm{ ~  for all ~  } t,s\in T,  \lambda\geq 0.
\end{equation} 
\end{definition}

The following is a technical assumption which holds in almost all cases of interest:
\begin{definition} [Separable process] \label{Separable process}
The random process $\{X_t\}_{t\in T}$ is called \emph{separable} if there is a countable set $T_0\subseteq T$ such that $X_t\in \lim_{\substack{s\rightarrow t \\ s\in T_0}} X_s$ for all $t\in T$  a.s.,
where $x\in \lim_{\substack{s\rightarrow t \\ s\in T_0}} x_s$ means that there is a sequence $(s_n)$ in $T_0$ such that $s_n\rightarrow t$ and $x_{s_n}\rightarrow x$.
\end{definition}
For instance, if $t\to X_t$ is continuous almost surely, then $X_t$ is a separable process (see e.g. \cite{Ramon}). 

Notice that, unlike a partition, an exact cover $\pa=\{A_i: i\in M\}$ of the set $T$ may have countably or uncountably infinite number of blocks, i.e. $M$ may have countably or uncountably infinite size.
\begin{definition}[$\epsilon$-cover]\label{epsilon cover definition} We call a cover $\pa=\{A_i: i\in M\}$ of the set $T$ an \emph{$\epsilon$-cover} of the metric space $(T,d)$ if for all $i\in M$, $A_i$ can be contained withing a ball of radius $\epsilon$. 
\end{definition} 
\begin{definition}[Hierarchical sequence of covers]
	A sequence of covers $\{\pa_k\}_{k=m}^{\infty}$ of a set $T$ is called a \emph{hierarchical sequence} (or an \emph{increasing sequence}) if for all $k\geq m$ and each $A\in\pa_{k+1}$, there exists $B\in \pa_k$ such that $A\subseteq B$. For any such sequence of exact covers and any $t\in T$, let $[t]_k$ denote the unique set $A\in \pa_k$ such that $t\in A$.
\end{definition}

If $\mathcal{N}$ is a set, let $X_{\mathcal{N}} \triangleq \{X_i : i \in N\}$ denote a random process indexed by the elements of $\mathcal{N}$.
For any bounded metric space $(S,d)$, let $k_1(S)$ be an integer such that $2^{-(k_1(S)-1)}\geq \mathrm{diam}(S)$.
\begin{theorem}\label{Generalization with chaining MI} Assume that $\{\mathrm{gen}(w)\}_{w\in \mathcal{W}}$ is a separable subgaussian process on the bounded metric space $(\mathcal{W},d)$. Let $\{\mathcal{P}_k\}_{k=k_1(\mathcal{W})}^{\infty}$ be a hierarchical sequence of exact coverings of $\mathcal{W}$, where for each $k\geq k_1(\mathcal{W})$, $\mathcal{P}_k$ is a $ 2^{-k}$-cover of $(\mathcal{W},d)$. 
\begin{enumerate}[(a)]
\item 
\begin{align}
&\mathrm{gen}(\mu, P_{W|S})\leq 3\sqrt{2} \sum_{k=k_1(\mathcal{W})}^{\infty}2^{-k}\sqrt{I([W]_k;S)},\nonumber
\end{align}
\item If $\mathbf{0}\in \{\ell(h_w,\cdot): w\in\mathcal{W}\}$, then
\begin{align}
&\mathrm{gen^+}(\mu, P_{W|S})\leq 3\sqrt{2} \sum_{k=k_1(\mathcal{W})}^{\infty}2^{-k}\sqrt{I([W]_k;S)+\log 2}, \nonumber
\end{align}
where $\mathbf{0}$ is a function identically equal to zero and $\gen^+(\mu,P_{W|S}) \triangleq \E\left[|L_{\mu}(W)- L_S(W)|\right]$. 
\end{enumerate}
\end{theorem}

Theorem \ref{Generalization with chaining MI} is in the context of statistical learning. The
more general counterpart in the context of random processes is Theorem \ref{Chaining MI Random Process}: 
\begin{theorem}\label{Chaining MI Random Process} Assume that $\{X_t\}_{t\in T}$ is a separable subgaussian process on the bounded metric space $(T,d)$. Let $\{\mathcal{P}_k\}_{k=k_1(T)}^{\infty}$ be a hierarchical sequence of exact coverings of $T$, where for each $k\geq k_1(T)$, $\mathcal{P}_k$ is a $ 2^{-k}$-cover of $(T,d)$. Let $W$ be a random variable taking values from $T$.
\begin{enumerate}[(a)]
\item
\begin{align}
\E[X_W]\leq 3\sqrt{2} \sum_{k=k_1(T)}^{\infty}2^{-k}\sqrt{I([W]_k;X_T)}.\nonumber
\end{align}
\item For any arbitrary $t_0\in T$,
\begin{equation}
\E[|X_W-X_{t_0}|]\leq 3\sqrt{2} \sum_{k=k_1(T)}^{\infty}2^{-k}\sqrt{I([W]_k;X_T)+\log 2}.\nonumber
\end{equation}
\end{enumerate}
\end{theorem}
\begin{proof}[Proof of Theorem \ref{Chaining MI Random Process}]
		
For an arbitrary $k\geq k_1(T)$, consider $\pa_k=\{A^{(k)}_i: i\in M_k\}$. Since $\pa_k$ is a $2^{-k}$-cover of $(T,d)$, based on Definition \ref{epsilon cover definition}, there exists a multi-set $\mathcal{N}_k\triangleq\{a_i:i\in M_k\}\subseteq T$ and a mapping $\pi_{\mathcal{N}_k}:T\to \mathcal{N}_k$ such that $\pi_{\mathcal{N}_k}(t)=a_i$ if $t\in A_i^{(k)}$ for all $i\in M_k$, and $d\left(t,\pi_{\mathcal{N}_k}(t)\right)\leq 2^{-k}$ for all $t\in T$. For an arbitrary $t_0\in T$, let $\mathcal{N}_{k_0}\triangleq \{t_0\}$. 
For any integer $n\geq k_1(T)$, we can write
\begin{equation}\nonumber
X_W=X_{t_0}+\sum_{k=k_1(T)}^n \left(X_{\pi_{\mathcal{N}_k}(W)}-X_{\pi_{\mathcal{N}_{k-1}}(W)}\right)+\left(X_W-X_{\pi_{\mathcal{N}_n}(W)}\right).
\end{equation}
Based on the definition of subgaussian processes, the process is centered, thus $\E [X_{t_0}]=0$. Therefore 
\begin{equation}\nonumber
\E [X_W]-\E \left[X_W-X_{\pi_{\mathcal{N}_n}(W)}\right]=\sum_{k=k_1(T)}^n \E \left[X_{\pi_{\mathcal{N}_k}(W)}-X_{\pi_{\mathcal{N}_{k-1}}(W)}\right].
\end{equation}
For every $k\geq k_1(T)$ and $t\in T$, based on the triangle inequality,
\begin{align}
d\left(\pi_{\mathcal{N}_k}(t),\pi_{\mathcal{N}_{k-1}}(t)\right)&\leq d\left(t,\pi_{\mathcal{N}_k}(t)\right)+d\left(t,\pi_{\mathcal{N}_{k-1}}(t)\right)\nonumber\\
&\leq 3\times 2^{-k}.\nonumber
\end{align} 
Knowing the value of $\left(\pi_{\mathcal{N}_k}(W),\pi_{\mathcal{N}_{k-1}}(W)\right)$ is sufficient to determine which one of the random variables $\left\{X_{\pi_{\mathcal{N}_k}(t)}-X_{\pi_{\mathcal{N}_{k-1}}(t)}\right\}_{t\in T}$ is chosen according to $W$. Therefore \\$\left(\pi_{\mathcal{N}_k}(W),\pi_{\mathcal{N}_{k-1}}(W)\right)$ is playing the role of the random index, and since $X_{\pi_{\mathcal{N}_k}(t)}-X_{\pi_{\mathcal{N}_{k-1}}(t)}$ is $d^2\left(\pi_{\mathcal{N}_k}(t),\pi_{\mathcal{N}_{k-1}}(t)\right)$-subgaussian, based on Theorem 2 of \cite{xu2017information}, an application of the data processing inequality and by summation, we have
\begin{equation}
\sum_{k=k_1(T)}^n\E \left[X_{\pi_{\mathcal{N}_k}(W)}-X_{\pi_{\mathcal{N}_{k-1}}(W)}\right]\leq \sum_{k=k_1(T)}^n 3\sqrt{2}\times 2^{-k}\sqrt{I(\pi_{\mathcal{N}_k}(W),\pi_{\mathcal{N}_{k-1}}(W);X_T)}.\nonumber
\end{equation}
Since $\{\pa_k\}_{k=k_1(T)}^{\infty}$ is a hierarchical sequence of coverings, for any $t\in T$, knowing $\mathcal{N}_k(t)$ will uniquely determine $\mathcal{N}_{k-1}(t)$. Therefore
\begin{align}
I\left(\pi_{\mathcal{N}_k}(W),\pi_{\mathcal{N}_{k-1}}(W);X_{T}\right)&=I\left(\pi_{\mathcal{N}_k}(W);X_{T}\right)\nonumber\\
							                                         &= I\left([W]_k;X_{T}\right).\nonumber
\end{align}
The rest of the proof follows from the definition of separable processes and the fact that 
\begin{equation}
    \lim_{n\to \infty} \E \left[X_W-X_{\pi_{\mathcal{N}_n}(W)}\right]=0.\nonumber
\end{equation}
\end{proof}
If in Theorem \ref{Chaining MI Random Process}, we let $T\triangleq \mathcal{W}$ and $X_w\triangleq \mathrm{gen}(w)$ for all $w\in \mathcal{W}$, then for each $k\geq k_1(T)$, due to the Markov chain 
\begin{equation}\label{the Markov chain}
X_T=\{\mathrm{gen}(w)\}_{w\in \mathcal{W}}\leftrightarrow S \leftrightarrow W \leftrightarrow [W]_k
\end{equation}
and the data processing inequality, we deduce $I([W]_k;X_T)\leq I([W]_k;S)$. Therefore Theorem \ref{Generalization with chaining MI} follows from Theorem \ref{Chaining MI Random Process}. 

If we use Theorem 2 of \cite{bu2019tightening} instead of Theorem 2 of \cite{xu2017information}, then we can tighten the bound of Theorem \ref{Generalization with chaining MI} to the following result. Recall that $S=(Z_1,\dots,Z_n)$ denotes the training set.

\begin{proposition} Assume that $\{\mathrm{gen}(w)\}_{w\in \mathcal{W}}$ is a separable subgaussian process on the bounded metric space $(\mathcal{W},d)$. Let $\{\mathcal{P}_k\}_{k=k_1(\mathcal{W})}^{\infty}$ be an increasing sequence of partitions of $\mathcal{W}$, where for each $k\geq k_1(\mathcal{W})$, $\mathcal{P}_k$ is a $ 2^{-k}$-partition of $(\mathcal{W},d)$. Then 
     	\begin{align}
&\mathrm{gen}(\mu, P_{W|S})\leq 3\sqrt{2} \sum_{k=k_1(\mathcal{W})}^{\infty}2^{-k}\left(\sum_{i=1}^n\sqrt{I([W]_k;Z_i)}\right),
\end{align}
\end{proposition}

\section{Proofs of generalization and excess risk bounds}\label{gen bounds section}
\begin{proof}[Proof of Lemma \ref{chaining link distance}] Since $\phi$ is $1$-Lipschitz and $\phi(0)=0$, for all vectors $x$ we have $|\phi(x)|_2\leq |x|_2$. Based on the triangle inequality, for all $1\leq i\leq k-1$, we can write 
\begin{align}
	\|\mathbf{W}_i\|_2&\leq \|\mathbf{W}_i-M_i\|_2+\|M_i\|_2\nonumber\\
					  &\leq (\alpha_i+1)\|M_i\|_2\nonumber\\
					  &\leq \exp(\alpha_i)\|M_i\|_2.\nonumber
\end{align}
Thus, for all $x^m\in \X$,
\begin{equation}
	|\sigma(\mathbf{W}_{k-1}(\dots \sigma(\mathbf{W}_1(x^m))\dots ))|_2\leq \exp\left(\sum_{i=1}^{k-1}\alpha_i\right)\left(\prod_{i=1}^{k-1}\|M_i\|_2\right)|x^m|_2. \nonumber
\end{equation}
This yields
\begin{align}
	|\sigma(\mathbf{W}_{k}(\dots(\sigma(\mathbf{W}_1(x^m))\dots)))-&\sigma(M_{k}(\dots(\sigma(\mathbf{W}_1(x^m))\dots)))|_2\nonumber\\
	&\leq \exp\left(\sum_{i=1}^{k-1}\alpha_i\right)\left(\prod_{i=1}^{k-1}\|M_i\|_2\right)\|\mathbf{W}_k-M_k\|_2|x^m|_2\nonumber\\
	&\leq \alpha_k \exp\left(\sum_{i=1}^{k-1}\alpha_i\right)\left(\prod_{i=1}^{k}\|M_i\|_2\right)|x^m|_2.\nonumber
\end{align}
Since $M_i$ is $\|M_i\|_2$-Lipschitz for all $k+1\leq i\leq d$, and soft-max is $1$-Lipschitz with respect to the Euclidean norm (see e.g. \cite{gao2017properties}), we conclude that
\begin{equation}
		|h_{w_1}(x^m)-h_{w_2}(x^m)|_2 \leq \alpha_k\exp\left(\sum_{i=1}^{k-1} \alpha_i\right) M|x^m|_2.\nonumber
	\end{equation}
\end{proof}
\begin{definition}
	For all $1\leq k \leq d$, let 
	\begin{equation}
	h_{[\mathbf{W}_1,\dots,\mathbf{W}_k]}\triangleq  h_{[\mathbf{W}_1,\dots,\mathbf{W}_k,M_{k+1},\dots,M_d]}\nonumber
	\end{equation}
and 
\begin{equation}
	\ell([\mathbf{W}_1,\dots,\mathbf{W}_k],z)\triangleq  \ell([\mathbf{W}_1,\dots,\mathbf{W}_k,M_{k+1},\dots,M_d],z).\nonumber
\end{equation}
\end{definition}
\begin{proof}[Proof of Theorem \ref{CMI generalization deep nets theorem}]

	Based on the Azuma--Hoeffding inequality, $\{\gen(w)\}_{w\in \mathcal{W}}$ is a subgaussian process with the metric
	\begin{equation}
	d(w, w')\triangleq \frac{\|\ell({w},\cdot)-\ell({w'},\cdot)\|_{\infty}}{\sqrt{n}},\nonumber
	\end{equation}
	regardless of the choice of distribution $\mu$ on $\mathsf{Z}$. For any example $z=(x^m,y)\in \mathsf{Z}$, we have 
\begin{equation}
	|\ell(w,z)-\ell(w',z)|\leq L\left|h_{w}(x^m)-h_{w'}(x^m)\right|_2.\nonumber
\end{equation}
	Therefore 
	\begin{equation}
		\|\ell(w,\cdot)-\ell(w',\cdot)\|_{\infty}\leq L \sup_{x^m\in\mathcal{X}} |h_{w}(x^m)-h_{w'}(x^m)|_{2}.\label{l infinity transformation}
	\end{equation}
	To apply the CMI technique, 
	we write the following chaining sum:
	\begin{align}
		\gen(W)&=\gen(W_1)+(\gen(W_1,W_2)-\gen(W_1))+\dots \nonumber\\
		&\quad +(\gen(W)-\gen(W_1,\dots,W_{d-1})).\nonumber
	\end{align}
	Taking expectations with respect to $P_{WS}$ yields 
	\begin{align}
		\gen(\mu, P_{W|S})= \E[\gen(W)]&=\E[\gen(W_1)]+\E[\gen(W_1,W_2)-\gen(W_1)]+\dots \nonumber\\
		&\quad +\E[\gen(W)-\gen(W_1,\dots,W_{d-1})],\label{expected value chaining sum appendix}
	\end{align}
	Based on Lemma \ref{chaining link distance},  for all $1\leq k \leq d$, we have
	\begin{equation}\label{hypotheses infinity bound}
		\sup_{x^m\in\mathcal{X}} |h_{[\mathbf{W}_1,\dots,\mathbf{W}_k]}(x^m)-h_{[\mathbf{W}_1,\dots,\mathbf{W}_{k-1}]}(x^m)|_{2}\leq \beta_k MR.	
	\end{equation}
	Using (\ref{l infinity transformation}), we deduce
	\begin{equation}\label{loss uniform bound}
		\|\ell([\mathbf{W}_1,\dots,\mathbf{W}_k],\cdot)-\ell([\mathbf{W}_1,\dots,\mathbf{W}_{k-1}],\cdot)\|_{\infty}\leq L\beta_k MR.
	\end{equation}
	Notice that knowing the value of $(W_1,\dots,W_k)$ is enough to determine which one of the random variables $\left\{\gen(\mathbf{W}_1,\dots,\mathbf{W}_k)-\gen(\mathbf{W}_1,\dots,\mathbf{W}_{k-1})\right\}_{w\in \W}$ is chosen according to $W$. Therefore $(W_1,\dots,W_k)$ is playing the role of the random index, and since $$\gen(\mathbf{W}_1,\dots,\mathbf{W}_k)-\gen(\mathbf{W}_1,\dots,\mathbf{W}_{k-1})$$ is $d^2\left([\mathbf{W}_1,\dots,\mathbf{W}_k],[\mathbf{W}_1,\dots,\mathbf{W}_{k-1}]\right)$-subgaussian, based on (\ref{loss uniform bound}), Theorem 2 of \cite{xu2017information} and an application of the data processing inequality on the Markov chain $\{\gen(w)\}_{w\in\W}\leftrightarrow S \leftrightarrow W \leftrightarrow (W_1,\dots,W_k)$, we obtain
	\begin{equation}
		\E[\gen(W_1,\dots,W_k)-\gen(W_1,\dots,W_{k-1})]\leq\frac{LMR\sqrt{2}\beta_k}{\sqrt{n}} \sqrt{I(S;W_1,\dots,W_k)}.\label{link mutual info bound appendix}
	\end{equation}
	From (\ref{expected value chaining sum appendix}) and (\ref{link mutual info bound appendix}) we deduce 
	\begin{equation}
		\gen(\mu,P_{W|S})=\E[\gen(W)]\leq \frac{LMR\sqrt{2}}{\sqrt{n}}\sum_{k=1}^d \beta_k \sqrt{I(S;W_1,\dots,W_k)}.\nonumber
	\end{equation}
\end{proof}

\begin{proof}[Proof of Theorem \ref{CMI excess risk}]
By plugging in $P_{W_1\dots W_k|S}\leftarrow \delta_{\widehat{w}_1\dots \widehat{w}_k}$ in the right side of (\ref{minimizing objective}), and by noting that $P_{W|S}^{\star}$ is defined as the conditional distribution which minimizes that expression, we obtain (\ref{minimizing objective 2}). 
\end{proof}

In the following, the notation $P_X\to Q_{Y|X}\to P_Y$ indicates that the joint distribution of $X$ and $Y$ is $P_{XY}=P_XQ_{Y|X}$. We state a high-probability result:
\begin{corollary}
    For a given $\mu$, let $\widehat{w}(\mu)$ denote the index of a hypothesis which achieves the minimum statistical risk among $\W$. If $P_S\to P_{W|S}^{\star}\to P_W$, then
\begin{equation}
	\pr\left[L_{\mu}(W)\leq\inf_{w\in\W}L_{\mu}(w)+ \epsilon\right]\geq 1-\frac{C}{\epsilon\sqrt{n}}\sum_{k=1}^d \beta_k\left(\gamma_k D\left(\delta_{\widehat{w}_1 \dots \widehat{w}_k}\middle\|Q^{(k)}_{W_1\dots W_k}\right)+\frac{1}{4\gamma_k}\right). \label{high probability ineq}
\end{equation}
\end{corollary} 
\begin{proof} 
Based on Theorem \ref{CMI excess risk}, we have
\begin{equation}
	\E[L_{\mu}(W)]-\inf_{w\in\W} L_{\mu}(w)\leq \frac{C}{\sqrt{n}}\sum_{k=1}^d \beta_k\left(\gamma_k D\left(\delta_{\widehat{w}_1 \dots \widehat{w}_k}\middle\|Q^{(k)}_{W_1\dots W_k}\right)+\frac{1}{4\gamma_k}\right). \nonumber
\end{equation}
Thus
\begin{equation}
	\E\left[L_{\mu}(W)-\inf_{w\in\W} L_{\mu}(w)\right]\leq \frac{C}{\sqrt{n}}\sum_{k=1}^d \beta_k\left(\gamma_k D\left(\delta_{\widehat{w}_1 \dots \widehat{w}_k}\middle\|Q^{(k)}_{W_1\dots W_k}\right)+\frac{1}{4\gamma_k}\right).\nonumber
\end{equation}
Since $L_{\mu}(W)-\inf_{w\in\W} L_{\mu}(w)$ is a positive random variable, by Markov's inequality we obtain
\begin{equation}
	\pr\left[L_{\mu}(W)-\inf_{w\in\W}L_{\mu}(w)> \epsilon\right]	\leq\frac{C}{\epsilon\sqrt{n}}\sum_{k=1}^d \beta_k\left(\gamma_k D\left(\delta_{\widehat{w}_1 \dots \widehat{w}_k}\middle\|Q^{(k)}_{W_1\dots W_k}\right)+\frac{1}{4\gamma_k}\right), \nonumber
\end{equation}
which yields
\begin{equation}
	\pr\left[L_{\mu}(W)\leq\inf_{w\in\W}L_{\mu}(w)+ \epsilon\right]\geq 1-\frac{C}{\epsilon\sqrt{n}}\sum_{k=1}^d \beta_k\left(\gamma_k D\left(\delta_{\widehat{w}_1 \dots \widehat{w}_k}\middle\|Q^{(k)}_{W_1\dots W_k}\right)+\frac{1}{4\gamma_k}\right).\nonumber
\end{equation}
\end{proof}

For the case of $\W$ being an arbitrary set, we state the following excess risk bound, whose proof is analogous to the proof of Theorem \ref{CMI excess risk}: 
\begin{proposition}\label{CMI excess risk arbitrary} Assume that $\W$ is an arbitrary set and for a given input distribution $\mu$, let $\widehat{w}(\mu)$ denote the index of a hypothesis which achieves the minimum statistical risk among $\W$. Let $B^{(\epsilon)}_{W_1\dots W_d}$ denote the uniform distribution over a neighborhood $U_{\epsilon}$ of $\widehat{w}(\mu)$ for which all $w\in U_{\epsilon}$ satisfy $L_{\mu}(w)\leq \inf_{w\in\W} L_{\mu}(w)+\epsilon$. Then
	\begin{equation}\label{CMI excess risk arbitrary ineq}
		\risk\left(\mu, P_{W|S}^{\star}\right)\leq \inf_{w\in\W}L_{\mu}(w)+\epsilon+\frac{C}{\sqrt{n}}\sum_{k=1}^d \beta_k\left(\gamma_k D\left(B^{(\epsilon)}_{W_1\dots W_k}\middle\|Q^{(k)}_{W_1\dots W_k}\right)+\frac{1}{4\gamma_k}\right).
\end{equation} 
\end{proposition}
\begin{proof}
By plugging in $P_{W_1\dots W_k|S}\leftarrow B^{(\epsilon)}_{W_1\dots W_k}$ in the right side of (\ref{minimizing objective}), and by noting that $P_{W|S}^{\star}$ is defined as the conditional distribution which minimizes that expression, we obtain (\ref{CMI excess risk arbitrary ineq}). 
\end{proof}

\section{Gibbs distribution results}\label{Gibbs algorithms section}

\begin{definition}[Gibbs distribution]\label{Gibbs distribution definition} The Gibbs (posterior) distribution associated to parameter $\gamma$ and prior distribution $Q$, is denoted with $P^{\gamma,Q}_{W|S}$ and defined as follows:
\begin{equation}
	P^{\gamma,Q}_{W|S=s}(\mathrm{d}w)\triangleq \frac{e^{-\gamma L_{s}(w)}Q(\mathrm{d}w) }{\E[e^{-\gamma L_{s}(\widetilde{W})}]}, \qquad \widetilde{W}\sim Q. \nonumber
\end{equation}
\end{definition}
\begin{lemma}\label{Gibbs optimizes}\cite{xu2017information}
	The Gibbs distribution $P^{\gamma,Q}_{W|S}$ is the unique solution to the optimization problem 
	\begin{equation}
		\argmin_{P_{W|S}}\left\{\E[L_S(W)]+\frac{1}{\gamma}D(P_{W|S}\|Q|P_S) \right\}.\nonumber
	\end{equation}
\end{lemma}

The next results are new excess risk bounds for the Gibbs distribution:
\begin{proposition}\label{Gibbs excess _ statistical learning}
	Assume that $\W$ is a countable set. For any input distribution $\mu$, let $\widehat{w}(\mu)$ denote the index of a hypothesis which achieves the minimum statistical risk among $\W$. If for all $w\in\W$, $\ell(w,Z)$ is $\sigma^2$-subgaussian where $Z \sim \mu$, then for any $\gamma>0$,
	\begin{equation}
		\risk\left(\mu, P_{W|S}^{\gamma, Q}\right)\leq \inf_{w\in\W} L_{\mu}(w)+\frac{1}{\gamma}D\left(\delta_{\widehat{w}(\mu)}\middle\|Q\right)+\frac{\gamma\sigma^2}{2n}.  \label{Gibbs excess risk ineq}
	\end{equation}
\end{proposition}
\begin{proof}
	Assuming $\mu^{\otimes n}=P_S\to P_{W|S}^{\gamma, Q} \to P_W$, we can write 
	\begin{align}
		\risk\left(\mu, P_{W|S}^{\gamma, Q}\right)&=\E[L_{\mu}(W)] \nonumber\\
												 &\leq \E[L_S(W)]+\sqrt{\frac{2\sigma^2}{n}}.\sqrt{I(S;W)}\nonumber\\
												 &\leq \E[L_S(W)]+\sqrt{\frac{2\sigma^2}{n}}\left(\frac{1}{\gamma}\sqrt{\frac{n}{2\sigma^2}}I(S;W)+\frac{1}{4\frac{1}{\gamma} \sqrt{\frac{n}{2\sigma^2}}} \right)\label{trick discrete}\\
												 &=\E[L_S(W)]+\frac{1}{\gamma}I(S;W)+\frac{\gamma\sigma^2}{2n}\nonumber\\
												 &\leq \E[L_S(W)]+\frac{1}{\gamma}D\left(P_{W|S}^{\gamma,Q}\middle\|Q\middle|P_S\right)+\frac{\gamma\sigma^2}{2n}\nonumber\\
												 &\leq \inf_{w\in\W} L_{\mu}(w)+\frac{1}{\gamma}D\left(\delta_{\widehat{w}(\mu)}\middle\|Q\right)+\frac{\gamma\sigma^2}{2n}\label{Gibbs minimizes},
	\end{align}
	where (\ref{trick discrete}) follows from the inequality 
	\begin{equation}
		\sqrt{x}\leq cx+\frac{1}{4c} \iff 0\leq\left(\sqrt{cx}-\frac{1}{2\sqrt{c}}\right)^2 \qquad \mathrm{for\ all\ } x,c > 0,\label{square root upper bound 2}
	\end{equation}
 which is upper bounding $\sqrt{x}$ with a tangent line, and (\ref{Gibbs minimizes}) follows from Lemma \ref{Gibbs optimizes} and by plugging $P_{W|S}\gets \delta_{\widehat{w}(\mu)}$ into 
 \begin{equation}
 	\E[L_S(W)]+\frac{1}{\gamma}D(P_{W|S}\|Q|P_S).\nonumber
 \end{equation}
 \end{proof}
 
\begin{corollary}
	If we set $\gamma \leftarrow \frac{1}{\sigma}\sqrt{2n D\left(\delta_{\widehat{w}(\mu)}\middle\|Q\right)}\triangleq \gamma^{\star}$, then we minimize the right side of (\ref{Gibbs excess risk ineq}) to obtain
	\begin{equation}
	    \risk\left(\mu, P_{W|S}^{Q, \gamma^{\star}}\right)\leq  \inf_{w\in\W} L_{\mu}(w)+\sigma\sqrt{\frac{D\left(\delta_{\widehat{w}(\mu)}\middle\|Q\right)}{2n}}.\nonumber
	\end{equation}
\end{corollary}

\begin{proposition}\label{Gibbs excess risk uncountable} Assume that $\W$ is an uncountable set. For any input distribution $\mu$, let $\widehat{w}(\mu)$ denote the index of a hypothesis which achieves the minimum statistical risk among $\W$. If for all $w\in\W$, $\ell(w,Z)$ is $\sigma^2$-subgaussian where $Z \sim \mu$ and $\ell(\cdot ,z)$ is $\rho$-Lipschitz for all $z\in \mathsf{Z}$, then for any $\gamma>0$,
	\begin{equation}
		\risk\left(\mu, P_{W|S}^{\gamma, Q}\right)\leq \inf_{w\in\W} L_{\mu}(w)+\inf_{a>0}\left(a\rho\sqrt{d} + \frac{1}{\gamma}D\left(\mathcal{N}\left(\widehat{w}(\mu), a^2I_d\right)\middle\|Q \right) \right)+\frac{\gamma\sigma^2}{2n},  \nonumber
	\end{equation}
	where $\mathcal{N}\left(\widehat{w}(\mu), a^2I_d\right)$ denotes the Gaussian distribution centered at $\widehat{w}(\mu)$ with covariance matrix $a^2I_d$.
\end{proposition}
\begin{proof}
	Assuming $\mu^{\otimes n}=P_S\to P_{W|S}^{\gamma, Q} \to P_W$, we can write 
	\begin{align}
		\risk\left(\mu, P_{W|S}^{\gamma, Q}\right)&=\E[L_{\mu}(W)] \nonumber\\
												 &\leq \E[L_S(W)]+\sqrt{\frac{2\sigma^2}{n}}.\sqrt{I(S;W)}\nonumber\\
												 &\leq \E[L_S(W)]+\sqrt{\frac{2\sigma^2}{n}}\left(\frac{1}{\gamma}\sqrt{\frac{n}{2\sigma^2}}I(S;W)+\frac{1}{4\frac{1}{\gamma} \sqrt{\frac{n}{2\sigma^2}}} \right)\label{trick}\\
												 &=\E[L_S(W)]+\frac{1}{\gamma}I(S;W)+\frac{\gamma\sigma^2}{2n}\nonumber\\
												 &\leq \E[L_S(W)]+\frac{1}{\gamma}D\left(P_{W|S}^{\gamma,Q}\middle\|Q\middle|P_S\right)+\frac{\gamma\sigma^2}{2n}\nonumber\\
												 &\leq \inf_{w\in\W} L_{\mu}(w)+
												 \inf_{a>0}\left(a\rho\sqrt{d} + \frac{1}{\gamma}D\left(\mathcal{N}(\widehat{w}(\mu), a^2I_d)\|Q \right) \right)+ \frac{\gamma\sigma^2}{2n}\label{Gibbs minimizes continuous},
	\end{align}
	where (\ref{trick}) follows from the inequality (\ref{square root upper bound 2}), 
 and (\ref{Gibbs minimizes continuous}) follows from Lemma \ref{Gibbs optimizes} and by plugging $P_{W|S}\gets \mathcal{N}\left(\widehat{w}(\mu),a^2I_d\right)$ into 
 \begin{equation}
 	\E[L_S(W)]+\frac{1}{\gamma}D(P_{W|S}\|Q|P_S),\nonumber
 \end{equation}
while writing
\begin{align}\label{Xu-Raginsky ineq}
    \E[L_S(W)]&\leq \inf_{w\in\W} L_{\mu}(w)+a\rho\sqrt{d}
\end{align}
and taking infimum over $a>0$. Inequality (\ref{Xu-Raginsky ineq}) is based on the proof of \cite[Corollary 3]{xu2017information}.
\end{proof}

More generally, in the context of empirical processes, let $\mathcal{F}=\{f_w: w\in \mathcal{W}\}$ be a collection of measurable functions from a set $\mathsf{Z}$ to $\mathbb{R}$, indexed by the set $\mathcal{W}$. Let $Z_1,Z_2,\dots, Z_n$ be a sequence of i.i.d elements drawn from $\mathsf{Z}$ with distribution $\mu$, and define $S=(Z_1,\dots,Z_n)$. For each $w\in\W$, define the empirical mean of function $f_w$ as 
\begin{equation}
	\mu_n(f_w)\triangleq \frac{1}{n}\sum_{i=1}^n f_w(Z_i), \nonumber
\end{equation}
and its true mean as 
\begin{equation}
	\mu(f_w)\triangleq \E \left[f_w(Z)\right],  \quad Z\sim \mu. \nonumber
\end{equation}
One can prove the following proposition, analogous to the poof of Proposition \ref{Gibbs excess _ statistical learning}:

\begin{proposition}
Assume that $\W$ is a countable set. For any input distribution $\mu$, let $\widehat{w}(\mu)$ denote the index of a function which has the minimum true mean among functions in $\mathcal{F}$. If $f_w(Z),\ Z\sim \mu$ is $\sigma^2$-subgaussian for all $w\in\W$, then for any $\gamma>0$,
	\begin{equation}
		\E\left[\mu(f_W)\right]\leq \inf_{w\in\W}\mu(f_w)+\frac{1}{\gamma}D(\delta_{\widehat{w}(\mu)}\|Q)+\frac{\gamma\sigma^2}{2n}, \nonumber
	\end{equation}
	where $\mu^{\otimes n}=P_S\to P^{\gamma, Q}_{W|S}\to P_W$.
\end{proposition} 

\section{Proof for the MT algorithm}\label{MT procedure proof section}
We first state the following lemmas. Lemma \ref{Tilted Renyi lemma} shows the useful role of tilted distributions in linearly combining relative entropies. For a proof, see \cite[Theorem 30]{van2014renyi}.
\begin{lemma}\label{Tilted Renyi lemma}
Let $\lambda\in [0,1]$. For any $P \ll Q$ and $P\ll R$,
\begin{equation}
	\lambda D(P\|Q)+(1-\lambda)D(P\|R)=D\left(P\|(Q,R)_{\lambda}\right)+(1-\lambda)D_{\lambda}(Q\|R),\nonumber
\end{equation}
where $(Q,R)_{\lambda}$ denotes the tilted distribution.	
Therefore 
\begin{equation}
	 \argmin_{P}  \left\{\lambda D(P\|Q)+(1-\lambda)D(P\|R)\right\}=(Q,R)_{\lambda},\nonumber
\end{equation}
and
\begin{equation}
	\min_P \left\{\lambda D(P\|Q)+(1-\lambda)D(P\|R)\right\}=(1-\lambda)D_{\lambda}(Q\|R).\nonumber
\end{equation}
\end{lemma}
The next lemma is a crucial property of conditional relative entropy:
\begin{lemma}\label{conditional relative entropy positive}
Given distribution $P_X$ defined on a set $\mathcal{A}$ and conditional distributions $P_{Y|X}$ and $Q_{Y|X}$, we have 
\begin{equation}
    D(P_{Y|X}\|Q_{Y|X}|P_X)\geq 0,
\end{equation}
with equality if and only if $P_{Y|X}=Q_{Y|X}$ holds on a set $\mathcal{A}'\subseteq \mathcal{A}$ of conditioning values with $P_X(\mathcal{A}')=1$.
\end{lemma}

The simplest case of (\ref{multilevel relative entropy sum minimization}) is when $d=2$, whose solution, characterized by the following result, is useful for obtaining the solution to the general case:
\begin{proposition}\label{two term multilevel sum prop}
	Let $Q_X$ and $R_{XY}$ be two arbitrary distributions. For any $a_1,a_2> 0$, we have	
	\begin{equation}\label{min problem simple}
		\argmin_{P_{XY}} \left( a_1 D(P_X\|Q_X)+a_2 D(P_{XY}\|R_{XY})\right)=P^{\star}_{XY},
	\end{equation}
	where 
	\begin{equation}\nonumber
	\begin{cases}
		P^{\star}_X=(Q_X,R_X)_{\frac{a_1}{a_1+a_2}},\\
				P^{\star}_{Y|X}=R_{Y|X}.
	\end{cases} 
	\end{equation}
\end{proposition}
\begin{proof} 
	Based on the chain rule of relative entropy, we have
	\begin{equation}
	    D(P_{XY}\|R_{XY})=D(P_X\|R_X)+D(P_{Y|X}\|R_{Y|X}|P_X). \nonumber
	\end{equation}
	Therefore 
	\begin{align}
		&a_1 D(P_X\|Q_X)+a_2 D(P_{XY}\|R_{XY})\nonumber\\&=a_1D(P_X\|Q_X)+a_2(D(P_X\|R_X)+D(P_{Y|X}\|R_{Y|X}|P_X))\nonumber\\
				&=(a_1D(P_X\|Q_X)+a_2D(P_X\|R_X))+a_2D(P_{Y|X}\|R_{Y|X}|P_X) \nonumber\\
				&=(a_1+a_2)D\left(P_X\middle\|(Q_X,R_X)_{\frac{a_1}{a_1+a_2}}\right)+a_2D_{\frac{a_1}{a_1+a_2}}(Q_X\|R_X)+a_2 D(P_{Y|X}\|R_{Y|X}|P_X),\label{Tilted Renyi equation}	
	\end{align}
	where (\ref{Tilted Renyi equation}) is based on Lemma \ref{Tilted Renyi lemma}. Note that, due to Lemma \ref{conditional relative entropy positive}, distribution $P^{\star}_{XY}$ is the unique distribution for which both relative entropies vanish simultaneously, and since the R\'{e}nyi divergence does not depend on $P_{XY}$, equation (\ref{min problem simple}) is proven.
\end{proof}
Inspired by the proof of Proposition \ref{two term multilevel sum prop}, we now give the proof of the general case: 
\begin{proof}[Proof of Theorem \ref{MT solution theorem}]
Similar to Proposition \ref{two term multilevel sum prop}, for the general case of arbitrary $d\geq 3$, we can solve (\ref{multilevel relative entropy sum minimization}) backwards and iteratively:
\begin{align}
	&\sum_{i=1}^d a_iD\left(P_{X_1\dots X_i}\|R^{(i)}_{X_1\dots X_i}\right)\nonumber\\
	&=\sum_{i=1}^{d-1} a_iD\left(P_{X_1\dots X_i}\middle\|R^{(i)}_{X_1\dots X_{i}}\right)\nonumber\\
	&\qquad+a_d\left(D\left(P_{X_1\dots X_{d-1}}\middle\|R^{(d)}_{X_1\dots X_{d-1}}\right)+D\left(P_{X_d|X_1\dots X_{d-1}}\middle\|R^{(d)}_{X_d|X_1\dots X_{d-1}}\middle|P_{X_1\dots X_{d-1}}\right)\right)\nonumber\\
	&=\sum_{i=1}^{d-2} a_iD\left(P_{X_1\dots X_i}\middle\|R^{(i)}_{X_1\dots X_{i}}\right)\nonumber\\
	&\quad + (a_{d-1}+a_d)D\left(P_{X_1\dots X_{d-1}}\middle\|\left(R_{X_1\dots X_{d-1}}^{(d-1)},R_{X_1\dots X_{d-1}}^{(d)}\right)_{\frac{a_{d-1}}{a_{d-1}+a_d}}\right)\nonumber\\
	&\quad + a_dD_{\frac{a_{d-1}}{a_{d-1}+a_d}}\left(R_{X_1\dots X_{d-1}}^{(d-1)}\middle\|R_{X_1\dots X_{d-1}}^{(d)} \right)\nonumber\\
	&\quad+a_d D\left(P_{X_d|X_1\dots X_{d-1}}\middle\|R^{(d)}_{X_d|X_1\dots X_{d-1}}\middle|P_{X_1\dots X_{d-1}}\right)\label{Tilted Renyi equation 2} ,
\end{align}
where (\ref{Tilted Renyi equation 2}) follows from Lemma \ref{Tilted Renyi lemma}. Notice that we can set 
\begin{equation}
	P^{\star}_{X_d|X_1\dots X_{d-1}}\gets R^{(d)}_{X_d|X_1\dots X_{d-1}}=S^{(d)}_{X_d|X_1\dots X_{d-1}},\nonumber
\end{equation}
to make the last conditional relative entropy in the right side of (\ref{Tilted Renyi equation 2}) vanish (and hence minimized, due to Lemma \ref{conditional relative entropy positive}), regardless of any choice for $P_{X_1\dots X_{d-1}}$ that we may take later on. Since the R\'{e}nyi divergence in (\ref{Tilted Renyi equation 2}) does not depend on $P_{X_1\dots X_d}$, we can ignore that term, and repeat this process to the sum of the remaining terms iteratively to obtain $P^{\star}_{X_i|X_1\dots X_{i-1}}=S^{(i)}_{X_i|X_1\dots X_{i-1}}$ for all $1\leq i\leq d-1$, where intermediate distributions $S^{(i)}_{X_1\dots X_i}$ are defined as in Algorithm \ref{MT}. In view of the fact that
\begin{equation}
    P^{\star}_{X_1\dots X_d}= P^{\star}_{X_1}P^{\star}_{X_2|X_1}\dots P^{\star}_{X_d|X_1\dots X_{d-1}},\nonumber
\end{equation}
we have obtained the desired distribution $P^{\star}_{X_1\dots X_d}$ as
\begin{equation}
    P^{\star}_{X_1\dots X_d}= S^{(1)}_{X_1}S^{(2)}_{X_2|X_1}\dots S^{(d)}_{X_d|X_1\dots X_{d-1}}.\nonumber
\end{equation}
\end{proof}
The key point of the previous proof is to rewrite the expression in (\ref{multilevel relative entropy sum minimization}) as the sum of some R\'{e}nyi divergences which do not depend on $P_{X_1\dots X_d}$, and some conditional relative entropies which can all be set equal to zero, simultaneously. Based on Lemma \ref{conditional relative entropy positive}, this happens if and only if $P_{X_1\dots X_d}= P^{\star}_{X_1\dots X_d}$, up to almost sure equality. 
\section{The multilevel Metropolis algorithm}\label{multilevel Metropolis appendix}
Using the MT algorithm, we derive the twisted distribution $P^{\star}_{W|S}$ for a two-layer net with prior distribution $Q^{(1)}_{W_1}$ and $Q^{(2)}_{W_1W_2}$, and temperature vector $(a_1,a_2)$, as
\begin{align}
	P^{\star}_{W|S=s}(w_1,w_2)
	&=\frac{\left(\bigintsss_{v_2}f(w_1,v_2)^{\frac{1}{a_2}} Q^{(2)}(w_1,v_2)\mathrm{d}v_2 \right)^{\frac{a_2}{a_1+a_2}}Q^{(1)}(w_1)^{\frac{a_1}{a_1+a_2}}}{\bigintsss_{v_1}\left(\bigintsss_{v_2}f(v_1,v_2)^{\frac{1}{a_2}}Q^{(2)}(v_1,v_2)\mathrm{d}v_2 \right)^{\frac{a_2}{a_1+a_2}}Q^{(1)}(v_1)^{\frac{a_1}{a_1+a_2}}\mathrm{d}v_1}\nonumber\\
	&\quad\times \frac{f(w_1,w_2)^{\frac{1}{a_2}}Q^{(2)}(w_1,w_2)}{\bigintsss_{v_2}f(w_1,v_2)^{\frac{1}{a_2}}Q^{(2)}(w_1,v_2)\mathrm{d}v_2}.\label{Two level twisted distribution}
\end{align}
In the case of having consistent product prior distributions $Q^{(1)}(w_1)=\tilde{Q}^{(1)}(w_1)$ and $Q^{(2)}(w_1,w_2)=\tilde{Q}^{(1)}(w_1)\tilde{Q}^{(2)}(w_2)$, equality (\ref{Two level twisted distribution}) simplifies to 
\begin{align}
	&P^{\star}_{W|S=s}(w_1,w_2)\nonumber\\
	&=\frac{\left(\bigintsss_{v_2}f(w_1,v_2)^{\frac{1}{a_2}} \tilde{Q}^{(2)}(v_2)\mathrm{d}v_2 \right)^{\frac{a_2}{a_1+a_2}}\tilde{Q}^{(1)}(w_1)}{\bigintsss_{v_1}\left(\bigintsss_{v_2}f(v_1,v_2)^{\frac{1}{a_2}}\tilde{Q}^{(2)}(v_2)\mathrm{d}v_2 \right)^{\frac{a_2}{a_1+a_2}}\tilde{Q}^{(1)}(v_1)\mathrm{d}v_1}\times \frac{f(w_1,w_2)^{\frac{1}{a_2}}\tilde{Q}^{(2)}(w_2)}{\bigintsss_{v_2}f(w_1,v_2)^{\frac{1}{a_2}}\tilde{Q}^{(2)}(v_2)\mathrm{d}v_2}. \nonumber
\end{align}

Notice that we can run line \ref{inner level} and line \ref{approximation step} of Algorithm \ref{multilevel Metropolis} concurrently, that is, each time we sample $v_2^{(i)}$, we can compute the next term in the sum in line \ref{approximation step}, hence the required space is a constant times the required space for storing matrices $w_1$ and $w_2$ and does not depend on the number of iterations. The computational complexity of the algorithm depends on the proposal distributions. The algorithm performs $T\times T'$ total iterations and at each of these iterations, the algorithm computes the empirical error over the entire training set.
\section{Experiment}\label{Experiments section}
The MNIST data set is available at \url{http://yann.lecun.com/exdb/mnist/}. This benchmark data set has 60000 training examples and 10000 test examples consisting of images with $28\times 28$ gray pixels and with $10$ classes. We flattened the images into vectors of length $784$ and normalized their values to between $0$ and $1$. Let $I_{m\times l}$ denote the $m\times l$ matrix with entries equal to $1$ on its main diagonal and zero elsewhere. We initialized the training algorithm at the reference matrices $M_1=I_{100\times 784}$ and $M_2=I_{10\times 100}$. For simplicity, we let the distributions $\tilde{Q}^{(1)}$ and $\tilde{Q}^{(2)}$ to be flat distributions, and we chose the temperature vector to be $\mathbf{a}=(2\times 10^{-6},10^{-6})$.   
The proposal distributions $q_1$ and $q_2$ are centered Gaussian distributions with independent entries having variances $0.001$ and $0.0005$, respectively. The training error at iteration $t=40000$ reached $0.052154361878265$ and the test error reached $0.066840303697749$. 

The computing infrastructure had the following specifications: 4.2 GHz Intel Core i7-7700K, 16 GB 2400 MHz DDR4 Memory, and Radeon Pro 575 4096 MB Graphics.

\section{Average predictors}\label{Gibbs average predictors}
\begin{definition}[Gibbs average predictor] We define the Gibbs average predictor as
\begin{equation}\nonumber
	h^{\gamma, Q}_s(x)\triangleq \E[h_W(x)], \qquad W\sim P^{\gamma, Q}_{W|S=s}.
\end{equation}
for all $s\in \Z^n$ and $x\in \X$, where $P^{\gamma, Q}_{W|S=s}$ is the Gibbs posterior distribution defined in Definition \ref{Gibbs distribution definition}. 
\end{definition}
Notice that the Gibbs average predictor is a deterministic function from $\X$ to $\mathcal{Y}$.
If $\ell(h,z)$ is convex in $h$, then based on Jensen's inequality,
\begin{equation}\label{Jensen based ineq}
	\ell(h^{\gamma,Q}_s,z)\leq \E[\ell(h_W,z)], \qquad W\sim P^{\gamma, Q}_{W|S=s}.
\end{equation}
Averaging both sides of (\ref{Jensen based ineq}) with respect to $Z\sim \mu$ and swapping the expectations on the right side gives
\begin{align}
	L_{\mu}\left(h^{\gamma,Q}_s\right)&\leq \E[L_{\mu}(W)], \qquad W\sim P^{\gamma, Q}_{W|S=s}.\nonumber
\end{align}
Taking expectations with respect to $P_S$ yields
\begin{align}
	\E \left[L_{\mu}\left(h^{\gamma,Q}_S\right)\right]&\leq \E[L_{\mu}(W)], \qquad P_S\to P^{\gamma, Q}_{W|S}\to P_W \nonumber\\
	&=\risk\left(\mu, P^{\gamma, Q}_{W|S} \right). \label{average Gibbs to Gibbs distribution}
\end{align}

Assume that $\mathcal{Y}=\{0,1\}$ and that the loss function is the $\ell_1$ loss. Based on the key idea of \cite[Equation (4.3)]{cesa1999prediction}, since $y$ can only take values $0$ or $1$, we have the following lemma:

\begin{lemma}
	If $\{h_w^{(k)}\}_{k=1}^d$ and $\{h_{w'}^{(k)}\}_{k=1}^d$ are collections of functions which take values from $\X$ to $[0,1]$, and $\xi_k>0$, $1\leq k\leq d$ are such that $\sum_{k=1}^d\xi_k=1$, then 
	\begin{equation}\label{absolute value loss identity}
	\left|\sum_{k=1}^d \xi_k h^{(k)}_w(x)-y\right|-\left|\sum_{k=1}^d \xi_k h^{(k)}_{w'}(x)-y \right|=\sum_{k=1}^d \xi_k \left[\left|h^{(k)}_w(x)-y\right|-\left|h^{(k)}_{w'}(x)-y\right|\right].
\end{equation}
\end{lemma}

\begin{corollary} Averaging both sides of (\ref{absolute value loss identity}) with respect to $z=(x,y)\sim \mu$ yields
\begin{equation}\label{absolute value loss identity 2}
	L_{\mu}\left(\sum_{k=1}^d \xi_k h^{(k)}_w\right)-L_{\mu}\left(\sum_{k=1}^d \xi_k h^{(k)}_{w'}\right)=\sum_{k=1}^d \xi_k \left(L_{\mu}\left(h^{(k)}_w\right)-L_{\mu}\left(h^{(k)}_{w'}\right)\right).
\end{equation}
\end{corollary}

Assume that $\W$ is a discrete set. We now construct an average predictor which achieves the excess risk bound of Theorem \ref{CMI excess risk}. For all $1\leq k\leq d$, let 
\begin{equation}\nonumber
	p_k\triangleq \frac{\beta_k}{\sum_{i=1}^d \beta_i}.
\end{equation}
Note that $\sum_{k=1}^d p_k=1$. For all $1\leq k\leq d$, let 
\begin{equation}\nonumber
	\mathcal{H}_k\triangleq \left\{\frac{1}{2}+\frac{h_{[\mathbf{W}_1,\dots,\mathbf{W}_k]}-h_{[\mathbf{W}_1,\dots,\mathbf{W}_{k-1}]}}{2\beta_kMR}:w\in \W \right\}.
\end{equation}
Based on inequality (\ref{hypotheses infinity bound}), the domain of all $h\in \mathcal{H}_k$ is $[0,1]$. Given training set $S$, let $h_S^{(k)}$ be the Gibbs average predictor obtained from $\mathcal{H}_k$ with prior $Q^{(k)}$ and inverse temperature 
\begin{equation}\nonumber
	\zeta_k\triangleq \frac{\sqrt{n}}{\gamma_k(\sum_{i=1}^d \beta_i)LMR}.
\end{equation} 
Based on (\ref{average Gibbs to Gibbs distribution}), the proof of Proposition \ref{Gibbs excess _ statistical learning}, and after taking average from both sides of (\ref{absolute value loss identity 2}) with respect to $P_S$, we get:
\begin{align*}
	\E\left[L_{\mu}\left(\sum_{k=1}^d p_k h^{(k)}_S\right) \right]-\inf_{w\in \W}&L_{\mu}(w)=\sum_{k=1}^d p_k \left(\E \left[L_{\mu}\left(h^{(k)}_S\right)\right]-\E \left[L_{\mu}\left(h^{(k)}_{\widehat{w}}\right)\right]\right)  \\ 
	&\leq LMR\sum_{k=1}^d p_k\left(\frac{\gamma_k(\sum_{i=1}^d \beta_i)}{\sqrt{n}}D\left(\delta_{\widehat{w}_1 \dots \widehat{w}_k}\middle\|Q^{(k)}_{W_1\dots W_k}\right)+ \frac{\sum_{i=1}^d\beta_i}{2\sqrt{n}\gamma_k} \right)\\
	&=\frac{C}{\sqrt{n}}\sum_{k=1}^d \beta_k\left(\gamma_k D\left(\delta_{\widehat{w}_1 \dots \widehat{w}_k}\middle\|Q^{(k)}_{W_1\dots W_k}\right)+\frac{1}{2\gamma_k} \right).
\end{align*}
 
\begin{remark}\normalfont
	The results of this section can be viewed as the ``dual" of the results of \cite{cesa1999prediction} in the supervised learning context. 
\end{remark} 

\end{document}